\documentclass{article}

\PassOptionsToPackage{dvipsnames}{xcolor}
\PassOptionsToPackage{numbers, compress}{natbib}

\usepackage[preprint]{neurips_2021}

\usepackage[utf8]{inputenc} 
\usepackage[T1]{fontenc}    
\usepackage{hyperref}       
\usepackage{url}            
\usepackage{booktabs}       
\usepackage{amsfonts}       
\usepackage{nicefrac}       
\usepackage{microtype}      
\usepackage{xcolor}         

\usepackage{microtype}
\usepackage{graphicx}
\usepackage{subcaption}
\usepackage{amsmath}
\usepackage{amssymb}
\usepackage{multirow}

\usepackage[normalem]{ulem}

\newcommand{\q}{\mathbf{q}}
\newcommand{\qd}{{\dot{\q}}}
\newcommand{\qdd}{{\ddot{\q}}}

\newcommand{\vv}{\mathbf{v}}

\newcommand{\x}{\mathbf{x}}
\newcommand{\xd}{{\dot{\x}}}

\newcommand{\J}{\mathbf{J}}
\newcommand{\Jd}{{\dot{\J}}}

\newcommand{\I}{\mathbf{I}}

\newcommand{\M}{\mathbf{M}}

\usepackage{amsmath}
\usepackage{amsfonts}
\usepackage{amssymb}
\usepackage{amsthm}
\usepackage{bm}
\usepackage{bbm}
\usepackage{mathtools}
\usepackage{enumitem}
\usepackage{thmtools,thm-restate}
\usepackage{algorithm}
\usepackage{algorithmic}
\usepackage{graphicx}
\usepackage{comment}
\usepackage[capitalise]{cleveref}

\theoremstyle{plain}

\newtheorem{theorem}{Theorem}[section]

\theoremstyle{definition}

\theoremstyle{remark}

\def\CC{\mathcal{C}}

\def\TT{\mathcal{T}}

\def\Jb{\mathbf{J}}\def\Lb{\mathbf{L}}

\def\ab{\mathbf{a}}

\def\Rbb{\mathbb{R}}

\def\R{\Rbb}

\def\t{\top}
\def\*{\star}

\title{Imitation Learning via Simultaneous Optimization of Policies and Auxiliary Trajectories}

\author{%
  Mandy Xie$^{1, 2}$\thanks{\texttt{manxie@gatech.edu}},
  Anqi Li$^{3}$\thanks{\texttt{anqil4@cs.washington.edu}},
  Karl Van Wyk$^{1}$\thanks{\texttt{kvanwyk@nvidia.com}},
  Frank Dellaert$^{2}$,
  Byron Boots$^{1, 3}$,
  Nathan Ratliff$^{1}$ \\
  $^{1}$ NVIDIA, 
  $^{2}$ Georgia Institute of Technology,
  $^{3}$ University of Washington
}

\begin{document}

\maketitle
\vspace{-7mm}
\begin{abstract}
  Imitation learning (IL) is a frequently used approach for data-efficient policy learning. Many IL methods, such as Dataset Aggregation (DAgger), combat challenges like distributional shift 
  by interacting with oracular experts. Unfortunately, assuming access to oracular experts is often unrealistic in practice; data used in IL frequently comes from offline processes such as lead-through or teleoperation. In this paper, we present a novel imitation learning technique called Collocation for Demonstration Encoding (CoDE) that operates on only a fixed set of trajectory demonstrations. We circumvent challenges with methods like back-propagation-through-time by introducing an auxiliary trajectory network, which takes inspiration from collocation techniques in optimal control. Our method generalizes well and more accurately reproduces the demonstrated behavior with fewer guiding trajectories when compared to standard behavioral cloning methods. We present simulation results on a 7-degree-of-freedom (DoF) robotic manipulator that learns to exhibit lifting, target-reaching, and obstacle avoidance behaviors.
\end{abstract}

\vspace{-4mm}
\section{Introduction}
\label{introduction}
\vspace{-3mm}
Programming sophisticated behavior for complex systems like robots 
can be time-consuming and complicated.  Recently, there has been significant effort directed toward circumventing the process of manually programming these systems and instead automatically synthesizing desired behavior using machine learning. 
Reinforcement learning (RL), in particular, has been proposed for learning policies in robotics~\cite{ibarz2021train,schulman2015high}, but these algorithms suffer from poor data efficiency, and require time-consuming exploration. 
Additionally, exploration on systems like real-world robots 
can be expensive, they are physical machines that consume energy and can cause damage when executing poorly selected actions. 
To contend with these challenges, roboticists often turn to \emph{imitation learning} (IL), for safer, more data-efficient policy learning. 

The simplest method for imitation learning is behavior cloning (BC) \cite{bain1999behaviorcloning,Ly20BehaviorCloning}, where policy learning is reduced to supervised action prediction on sequences of demonstrated states. BC can be challenging, as observed in~\citet{Ross11Dagger}, because one-step action prediction error can accumulate causing policy rollouts to diverge from the demonstrations over time. 
This problem is known as {\em covariate shift}~\cite{Ross11Dagger}, 
and has been addressed via reductions to online learning, where an oracular expert 
is queried to provide demonstrations at newly encountered states when following the policy~\cite{Ross11Dagger,ross2014aggrevated,sun2017DeeplyAggrevated}. Unfortunately, access to an oracle is often unrealistic when demonstrations are expensive. 

Despite these known limitations, BC continues to be a popular approach in practice due to its simplicity. Moreover, recent work has shown that BC can achieve near-expert performance on a set of standardized tasks with RL experts~\cite{spencer2021feedback}. 
This inspires us to further scrutinize how errors are generated and accumulated within BC. 

We observe that an important source of one-step prediction error is \emph{inconsistencies} between the demonstrations and the learner, with examples including observation noise, dynamics mismatch, and limited expressivity of the learner's policy class. 
These errors are \emph{inherent to the demonstrations and the learner}, and often cannot be reduced during learning. For example, smooth policy classes are often used to learn robotics tasks since jerky motion can be harmful to robot hardware. However, demonstrations may not be smooth due to measurement noise, communication delay during teleoperation, and so on. In cases like these, the BC policy rollout may follow the demonstration well until it encounters a sharp turn, and will then diverge from the demonstration (see~\cref{fig:CoDE_illustration} (left)). 

Taking inspiration from collocation methods in optimal control~\cite{Hargraves87jgcd_collocation, Hereid16icra_collocationHumanoid, Posa16icra_collocationKinematic}, we propose a novel imitation learning technique called Collocation for Demonstration Encoding (CoDE). The key idea behind CoDE is to provide the BC learner a set of auxiliary trajectories that are near the demonstrations, while ensuring that the one-step prediction error is small for \emph{all} states along the auxiliary trajectories. During learning, we jointly optimize both the auxiliary trajectory and the policy: We minimize both the one-step prediction error with respect to the auxiliary trajectory as well as the deviation between the auxiliary trajectory and the demonstration (see~\cref{fig:CoDE_illustration} (middle, right)). 
Although CoDE requires the availability of a dynamics model, it can work with state-only demonstrations, which can be more easily obtained during practice (see Section 3 for a discussion).

CoDE can be viewed as formalizing a group of practical fixes to BC, such as smoothing the demonstrations~\cite{argall2009survey}, data augmentation~\cite{bojarski2016visualbackprop}, and noise injection~\cite{laskey2017dart,ke2020grasping}. These practical techniques either explicitly generate auxiliary trajectories which hopefully match well with the learner's policy class~\cite{rana2020benchmark}, or augment the demonstrations so that there exists such trajectories within its coverage, e.g. through noise injection~\cite{laskey2017dart,ke2020grasping}. However, these informal techniques are highly task-dependent. 
In contrast, CoDE fundamentally contends with the inconsistency 
by explicitly \emph{learning} a set of auxiliary trajectories that are consistent with the policy class. 

We provide theoretical intuitions on how CoDE is able to improve policy rollout performance over BC, 
and we demonstrate that CoDE can outperform both BC and BC with noise injection~\cite{ke2020grasping}, a common practical fix to BC, on two imitation learning tasks for a simulated 7-DoF Franka robot. 



\begin{figure}[!tb]
	\centering
	\vspace{-4mm}
    \includegraphics[width=1\columnwidth]{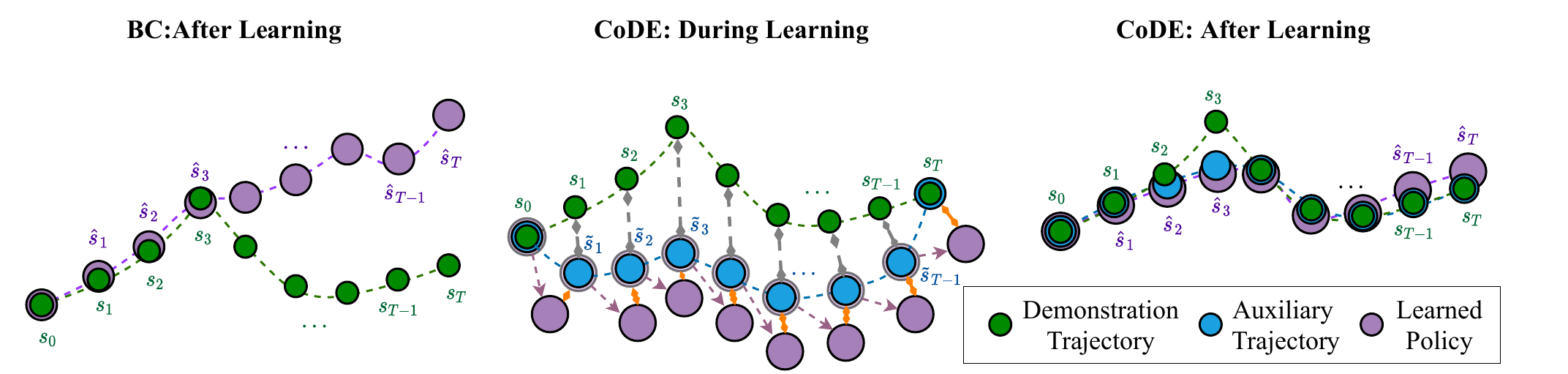}
    
	\caption{Intuition for CoDE. The green, blue and purple trajectories represent the demonstration, auxiliary, and policy rollout trajectories respectively. 
	\textbf{Left:} Learning to imitate a non-smooth demonstration using a smooth policy through \emph{BC}. The policy rollout (purple) diverges from the demonstration (green) after the sharp turn. \textbf{Middle:} In CoDE, we additionally parameterize an auxiliary trajectory. During learning, we simultaneously minimize the deviation between the auxiliary trajectory and the demonstration (gray arrows) and the one-step prediction loss of the learned policy with respect to the auxiliary trajectory (orange arrows). The purple arrows represent the actions predicted by the learned policy from states along the auxiliary trajectory $\{\tilde{s}_t\}_t$. \textbf{Right:} After convergence, the auxiliary trajectory stays close to the demonstration while being compatible with the policy class. 
	The rollout trajectory of the learned policy closely track both the auxiliary trajectory and the demonstration.}
	\label{fig:CoDE_illustration}
	\vspace{-6mm}
\end{figure}

\vspace{-4mm}
\section{Related work}
\vspace{-3mm}
Despite its popularity and simplicity, behavior cloning \cite{bain1999behaviorcloning,Ly20BehaviorCloning,PomerleauALVINN88} suffers from covariate shift~\cite{Ross11Dagger}. 
Partly due to this, 
there has been extensive work on Inverse Reinforcement Learning (IRL)~\cite{NgRussellIRL2000,AbbeelNgIRL2004}, including its reduction to structured prediction~\cite{MMPRatliff2006,StructuredPrediction2007} and related entropy maximization techniques~\cite{MaxEntIOCZiebart2008}. These algorithms require planners in the inner loop which adds both computational complexity and requires the existence of such globally optimal planners. There has been work \cite{LocallyOptimalIOCLevine2012,englert2017InverseKKT} into leveraging KKT conditions and optimization theory,
but all these methods focus on exploiting the structure of the MDP rather than directly training a policy.

Adversarial Imitation Learning (AIL) such as GAIL~\cite{Ho16arxiv_GAIL} and AIRL~\cite{Fu17arxiv_AIRL}, which frames imitation learning as learning a policy that minimizes a divergence between the state-action distributions of the expert trained policy, has also recently become a popular approach for imitation learning. These methods are sample inefficient since a large number of policy interactions with the environment are required and inherit training stability challenges characteristic of adversarial learning making them difficult to apply in practice to robot learning~\cite{spencer2021feedback}.

A separate thread of research introduced SMiLE \cite{ross2010Smile} and DAgger \cite{Ross11Dagger}, with later extensions \cite{ross2014aggrevated,sun2017DeeplyAggrevated}. These algorithms 
directly address the covariate shift issue both theoretically and pragmatically using a technique called dataset aggregation. However, these algorithms assume access to an oracle that can be queried, which is an stronger assumption than BC, 
since BC leverages only data collected \emph{a priori}. 

\citet{Venkatraman14DaD} focused in particular on learning dynamical systems, and devised a method of using the data itself as an oracular expert. This method, termed Data as Demonstrator (\textsc{DaD}), alleviates the covariate shift issue by generating synthetic data to guide the prediction back to demonstrated states. 
\textsc{DaD} is perhaps the most closely related to our technique CoDE in that CoDE also uses the data to effectively generate new BC problems that are easier to solve. 
However, \textsc{DaD} may generate conflicting demonstrations, an issue the authors acknowledge. In contrast, our technique is able to generate \emph{consistent} auxiliary demonstrations for the learner given its inductive bias. 
Moreover, our technique is more practical than \textsc{DaD} since our technique only solves a single optimization problem, while \textsc{DaD} requires solving a number of rounds of optimization problems generated by dataset aggregation.


\vspace{-3mm}
\section{Problem Statement}\label{sec:problem}
\vspace{-3mm}
We address the problem of imitation learning from trajectory demonstrations. In this section, we introduce the problem setup and discuss two common solutions, behavior cloning \mbox{and optimal tracking.} 

Consider a deterministic discrete-time system with known dynamics model $s_{t+1}=f(s_t, a_t)$
where $s_t\in\R^n$ is the state, and $a_t\in\R^m$ is the action of the system, and $f:\R^n\times\R^m\to\R^n$ is the system dynamics. 
We want to learn a reactive policy $\pi_\theta:s\mapsto\pi_\theta(s)$ from a restricted policy class $\Pi:=\{\pi_\theta:\theta\in\Theta\}$ given $N$ trajectory demonstrations $\{\tau^{(i)}\}_{i=1}^N$. 
Each demonstration is defined as a sequence of states, 
denoted $\tau^{(i)}:=\{s_t^{(i)}\}_{t=0}^{T_i}$, where $T_i$ is the horizon for the $i$th demonstration. We do not assume that the sequence of actions is given to the learner. 

\textbf{Remark.}
There are two major distinctions between our assumptions and common assumptions in prior work, e.g. \cite{rana2020benchmark,PomerleauALVINN88,Ho16arxiv_GAIL,Fu17arxiv_AIRL}:  \emph{1)} we assume known dynamics model;
\emph{2)} we assume that only state trajectories are available, instead of state-action trajectories commonly assumed in prior work. 
At the first glance, the assumption of known dynamics is stronger. However, when only states (not actions) are observed, having some knowledge on the dynamics is necessary. In fact, a common practice 
is to first extract actions from state trajectories using an inverse model, e.g. obtain velocity and acceleration actions through finite differencing~\cite{rana2020benchmark}, and then apply model-free algorithms such as BC. In~\cref{sec:instantiation}, we present an instantiation of our proposed technique with the same idea of extracting actions from an inverse model. 
%
%
Importantly, our analysis can be generalized to when only an approximate dynamics model is available, which could  be learned from data
\mbox{(see~\cref{fnote:approx}).}

We now describe two common strategies which we combine in Section~\ref{sec:CoDE} into our CoDE algorithm. The first one is a variant of behavior cloning (BC) with state-only demonstrations; 
the second solves an 
optimal control problem that addresses long-horizon performance.

\textbf{Behavior cloning.}
Suppose we are additionally given the actions which produces the trajectories, i.e., $a_t^{(i)}$ such that, $s_{t+1}^{(i)} = f(s_t^{(i)}, a_t^{(i)}),\,\forall i\in\{1,\ldots, N\},\, t\in\{0,\ldots,T_i-1\}.
$ One way to solve for a policy is to directly regress on actions along the demonstration trajectories, known as \emph{behavior cloning}~\cite{bain1999behaviorcloning,Ly20BehaviorCloning}. Formally, BC solves the following optimization problem:
    $\underset{\theta}{\min}\,
    \sum_{i=1}^N\sum_{t=0}^{T_i-1}\ell(a_t^{(i)}, \pi_\theta(s_t^{(i)})),$
where $\ell(\cdot,\cdot)$ is a symmetric and smooth 
loss function. 

However, the sequence of actions $\{\{a_t^{(i)}\}_t\}_i$ are not given under our setup. In this case, we can make use of the dynamics model ${f}$ and instead minimize the one-step prediction error, i.e.
\begin{equation}\label{eq:bc_states}\small
    \underset{\theta}{\min}\, 
    \sum_{i=1}^N\sum_{t=0}^{T_i-1}\ell\left(s_{t+1}^{(i)}, {f}(s_t^{(i)},\pi_\theta(s_t^{(i)}))\right).
\end{equation}

Though appealing due to its simplicity, BC suffers from \emph{covariate shift}~\cite{Ross11Dagger,Venkatraman14DaD}, namely, cascading error over time when the learned policy is executed. Here, we restate a theorem from~\cite{Venkatraman14DaD} on the error bound for BC\footnote{When given an inaccurate dynamics model $\tilde{f}$ such that $\|f-\tilde{f}\|_\infty\leq \eta$, we can obtain a bound of $O(e^{(T_i+1)\log(L)}(\epsilon+\eta))$ through the same analysis. Similarly for~\cref{thm:code}.\label{fnote:approx}}, which follows directly from Theorem 1 in~\cite{Venkatraman14DaD} with $\widehat{M}:=f(\cdot, \pi_\theta(\cdot))$. 

\begin{theorem}[\cite{Venkatraman14DaD}]\label{thm:bc}
Assume the map $f(\cdot, \pi_\theta(\cdot))$ is $L$-Lipschitz with respect to norm $\|\cdot\|$, i.e., for any states $s$, $s'$,
    $\|f(s, \pi_\theta(s)) - f(s', \pi_\theta(s'))\| \leq L\|s-s'\|,$
Suppose that $L>1$, if one-step prediction error satisfies
    $\|s_{t+1}^{(i)} - {f}(s_t^{(i)},\pi_\theta(s_t^{(i)}))\|\leq\epsilon$
for all $i\in\{1,\ldots,N\}$, $t\in\{1,\ldots,T_i-1\}$, then
    $\|\hat{s}_t^{(i)}-s_t^{(i)}\|\leq \sum_{\tau=0}^{t-1}L^\tau\epsilon\in O(e^{t\log(L)}\epsilon),$
where $\{\hat{s}_t^{(i)}\}_{t=1}^{T_i-1}$ is the rollout trajectory of $\pi_\theta$ under dynamics $f$ initialized at $s_0^{(i)}$. 
Hence, the trajectory deviation measured in norm  $\|\cdot\|$ is bounded by,
\begin{equation}\label{eq:bc_bound}\small\textstyle
    \sum_{t=0}^{T_i}\|\hat{s}_t^{(i)}-s_t^{(i)}\| = O(e^{T_i\log(L)}\epsilon).
\end{equation}
\end{theorem}
In general, it is hard to show that the Lipschitz constant $L\leq 1$ (in which case the error bound would be linear or sublinear to $T_i$) unless the policy class $\Pi$ has certain stability properties, e.g.~\cite{khansari2011learning}. However, this would also imply that the closed loop system (under the learned policy) has a contraction property and all trajectories must converge, which may not be desirable for many tasks. 

By~\cref{thm:bc}, the trajectory deviation is proportional to the one-step prediction error $\epsilon$ multiplied by a term which is exponentially increasing with respect to the task horizon $T_i$. 
Although the BC objective seeks to directly minimize the one-step prediction error $\epsilon$, the resulting error might still be non-trivial due to the reasons listed below (as a result, the rollout error becomes even larger):
\begin{enumerate}[leftmargin=*,itemsep=0ex, parsep=0.2ex]
\item \textit{Optimization error.} The algorithm was not able to converge to the global optimal solution of~\eqref{eq:bc_states}. 
\item \textit{Noise in state observation.} There could be measurement noise in the state demonstrations. 
\item \textit{Dynamics mismatch between the expert and the learner.} The demonstration may not be collected under the same dynamics. For example, the demonstration can be collected from human achieving a task, while the learner needs to find a policy for a robot manipulator. The demonstration might also have access to additional actions that are not available to the learner. 
\item \textit{Limitation of policy class.}
In practice, policy classes often have inductive bias that allow learned policies to generalize to unseen data and preserve certain desirable properties, e.g. smoothness. In the case when such properties are not satisfied by the demonstrations, e.g. non-smooth demonstrations for smooth policy classes, the one-step prediction error will inevitably be large. 
\end{enumerate}

\textbf{Optimal tracking.}
On the other side of the spectrum, optimal tracking and multi-step prediction~\cite{Venkatraman14DaD,anderson2007optimal,abbeel2005learning,langford2009learning,werbos1990backpropagation}, directly optimizes for the deviation from the trajectory demonstration and the rollout trajectory of the learned policy under the dynamics function $f$: 
\begin{align}\small
\min_{\theta,\{\hat{s}_{t}^{(i)}\}} \quad & 
\sum_{i=1}^N\sum_{t=0}^{T_i}\ell\big(s_t^{(i)}, \hat{s}_{t}^{(i)}\big)\label{eq:collocation-objective}\\
\textrm{s.t.\quad} &\hat{s}_{t+1}^{(i)} = {f}\big(\hat{s}_{t}^{(i)},\pi_\theta(\hat{s}_{t}^{(i)})\big)\qquad\forall i\in\{1,\ldots, N\}, t\in\{0,\ldots,T_i-1\},\label{eq:collocation-constraint}\\
  &\hat{s}_{0}^{(i)} = s_0^{(i)},\qquad \forall i\in\{1,\ldots, N\}\label{eq:collocation-init}.
\end{align}
The constraints~\eqref{eq:collocation-constraint}--\eqref{eq:collocation-init} ensure that $\{\hat{s}_t^{(i)}\}$ is the rollout trajectory of policy $\pi_\theta$ under dynamics ${f}$ initialized at $s_0^{(i)}$. 
This problem~\eqref{eq:collocation-objective}--\eqref{eq:collocation-init} could be potentially solved by back-propagation-through-time~\cite{werbos1990backpropagation,langford2009learning}. However, back-propagation-through-time often suffers from vanishing gradients for general dynamics functions and policy classes~\cite{Venkatraman14DaD}. 
There is also a fundamental difference between optimal tracking~\eqref{eq:collocation-objective}--\eqref{eq:collocation-init} and imitation learning: Optimal tracking only is only concerned with policy performance with respect to the given demonstrations, i.e., the training set, while imitation learning focuses on how the learned policy \emph{generalizes} to unseen data. 

\vspace{-4mm}
\section{Collocation for Demonstration Encoding}
\label{sec:CoDE}
\vspace{-3mm}
BC focuses on the problem of single-step action prediction and is myopic. As mentioned in~\cref{sec:problem}, in some cases, the action prediction errors are inevitable due to the inconsistency between the demonstrations and the policy class, e.g. measurement noise, dynamics mismatch, and inductive bias. 

We present our approach, Collocation for Demonstration Encoding (CoDE), for imitation learning which provides a principled way to mitigate the inconsistency between demonstrations and the policy class. 
%
%
The key idea behind CoDE is that we can present the BC policy learner with an auxiliary imitation problem more compatible with the policy class 
while ensuring this auxiliary problem still tracks the desired demonstrated trajectory well. 
In that way, the rollout error can be significantly reduced so that its effect on trajectory deviation is (almost) linear with respect to the problem horizon (rather than exponential, as in \eqref{eq:bc_bound}). To achieve this, we take inspiration from collocation methods~\cite{Hargraves87jgcd_collocation,Hereid16icra_collocationHumanoid,Posa16icra_collocationKinematic}, which have been successfully applied to motion planning and optimal control. Concretely, in addition to the policy, we also explicitly introduce parameterized trajectories which approximates the rollouts of the learned policy, which we call \textit{auxiliary trajectories}.

Formally, given auxiliary trajectories parameterized by $\phi$, denoted $\{\tilde{s}_{t,\phi}^{(i)}\}$. We explicitly parameterize the auxiliary trajectories $\{\tilde{s}_{t,\phi}^{(i)}\}$ such that its initial states match with the trajectory demonstrations, i.e., $\tilde{s}_{0,\phi}^{(i)} = s_0^{(i)},$ for all $i\in\{1,\ldots, N\}$, and solve the following optimization problem:
\begin{equation}\label{eq:collocation-lagrangian}\small
    \begin{split}
        &\min_{\theta,\phi} \; L(\theta,\phi;\lambda):=
        \sum_{i=1}^N\sum_{t=0}^{T_i}\ell\big(s_t^{(i)}, \tilde{s}_{t,\phi}^{(i)}\big) + \lambda\,
        \sum_{i=1}^N\sum_{t=0}^{T_i-1}\ell\Big(\tilde{s}_{t+1,\phi}^{(i)},  {f}\big(\tilde{s}_{t,\phi}^{(i)},\pi_\theta(\tilde{s}_{t,\phi}^{(i)})\big)\Big),
    \end{split}
\end{equation}
where $\ell(\cdot,\cdot)$ is a symmetric, smooth and continuously differentiable loss function with $\ell(s,s')=0$ if and only if $s=s'$, which measures state deviation. 
The objective function~\eqref{eq:collocation-lagrangian} is a weighted sum of \emph{1)} the deviation between the trajectory demonstrations $\{s_t^{(i)}\}$ and the auxiliary trajectory rollouts $\{\tilde{s}_{t,\phi}^{(i)}\}$, and \emph{2)} the one-step prediction error of the learned policy $\pi_\theta$ with respect to the \emph{auxiliary trajectories} $\{\tilde{s}_{t,\phi}^{(i)}\}$. The multiplier $\lambda>0$ is a hyper-parameter which trades-off the two terms. 

The idea of CoDE is illustrated in Fig.~\ref{fig:CoDE_illustration}, 
where trajectories in green, blue and purple represent demonstration trajectories, auxiliary trajectories, and the rollout trajectories of the learned policy, respectively (all starting from the same initial state). The middle figure shows that, during the learning process, we  minimize both \emph{1)} the deviation between the demonstration trajectory and the auxiliary trajectory and \emph{2)} the one-step prediction error between the learned policy along the auxiliary trajectory. The right figure shows that, after learning, the learned policy rollout trajectory more closely matches both the auxiliary and demonstration trajectories.

\textbf{Connection to optimal tracking.} Note that, when taking  $\lambda\to\infty$ in~\eqref{eq:collocation-lagrangian} and assuming that $\{\tilde{s}_{t,\phi}^{(i)}\}$ is expressive enough, due to the equivalence of the Lagrangian and the constraint formulation, we must have $i\in\{1,\ldots, N\}, t\in\{0,\ldots,T_i-1\}$, $\tilde{s}_{t+1,\phi}^{(i)} = \tilde{f}\big(\tilde{s}_{t,\phi}^{(i)},\pi_\theta(\tilde{s}_{t,\phi}^{(i)})\big)$. 
Under such condition, we recover the optimal tracking problem~\eqref{eq:collocation-objective}--\eqref{eq:collocation-init} (with our assumption that $\tilde{s}_{0,\phi}^{(i)} = s_0^{(i)}$). Compared to optimal tracking, our approach solves a \textit{well-conditioned} optimization problem~\eqref{eq:collocation-lagrangian}, which can be efficiently applied to a more general class of dynamics functions and policy classes. Our analysis is also applicable to when the auxiliary trajectories are from a restricted function class, in which case the equality constraint~\eqref{eq:collocation-constraint} may not be feasible. 

\textbf{Connection to behavior cloning. }Suppose that the auxiliary trajectories $\{\tilde{s}_{t,\phi}^{(i)}\}$ are expressive enough, we can potentially perfectly fit the demonstrated trajectories with the auxiliary trajectories so that  $\tilde{s}_{t,\phi}^{(i)} = s_t^{(i)}$ for all $t$ and $i$. If we keep the auxiliary trajectory fixed at $\tilde{s}_{t,\phi}^{(i)} = s_t^{(i)}$ while optimizing for the policy $\pi_\theta$ in~\eqref{eq:collocation-lagrangian}, we recover behavior cloning on the demonstrations (as the first term is $0$ by the choice of $\phi$). In fact, the policy is essentially solving the behavior cloning problem with respect to the auxiliary trajectories $\{\tilde{s}_{t,\phi}^{(i)}\}$ instead of the demonstrations $\{s_{t}^{(i)}\}$. By additionally parameterizing the auxiliary trajectories, our approach is able to \emph{separate out the inconsistency between the demonstrations and the policy class (the first term in~\eqref{eq:collocation-lagrangian}) from the behavior cloning error (the second term)}. 
We hereby state the main theorem for this paper.
\begin{theorem}\label{thm:code}
Assume that the map $f(\cdot, \pi_\theta(\cdot))$ is $L$-Lipschitz with respect to norm $\|\cdot\|$ with $L>1$. 
Suppose the following error bounds hold:
\vspace{-2mm}
\begin{enumerate}[leftmargin=*,itemsep=0.2ex]
    \item Uniformly bounded one-step error w.r.t. the auxiliary trajectories, i.e.,
    $\|\tilde{s}_{t+1,\phi}^{(i)} - {f}(\tilde{s}_t^{(i)},\pi_\theta(\tilde{s}_t^{(i)}))\|\leq\delta,$ for all $i\in\{1,\ldots,N\}$, $t\in\{1,\ldots,T_i-1\}$. 
    \item Uniformly bounded deviation between the demonstrated trajectories and the auxiliary trajectories,
    $\|s_t^{(i)}-\tilde{s}_{t,\phi}^{(i)}\|\leq \kappa$, for all $i\in\{1,\ldots,N\}$, $t\in\{1,\ldots,T_i-1\}$. 
\end{enumerate}
Then the trajectory deviation between the demonstrated trajectories and the rollout trajectories of $\pi_\theta$ measured in $\|\cdot\|$ is bounded by,
\begin{equation}\label{eq:code-bound}\small\textstyle
    \sum_{t=0}^{T_i}\|\hat{s}_t^{(i)}-s_t^{(i)}\| = O(\kappa T_i + e^{T_i\log(L)}\delta),
\end{equation}
where $\{\hat{s}_t^{(i)}\}_{t=1}^{T_i}$ is the rollout trajectory of $\pi_\theta$ under dynamics $f$ initialized at $s_0^{(i)}$. 
\end{theorem}
\vspace{-4mm}
\begin{proof}
By the triangle inequality, and that the demonstration, auxiliary trajectory, and the policy rollout are initialized at the same state, i.e., $\tilde{s}_{0,\phi}^{(i)}=\hat{s}_0^{(i)}=s_0^{(i)}$,
\begin{equation*}\small
    \begin{split}
        \sum_{t=0}^{T_i}\left\|\hat{s}_t^{(i)}-s_t^{(i)}\right\| &\leq \sum_{t=1}^{T_i}\left\|\tilde{s}_{t,\phi}^{(i)}-s_t^{(i)}\right\| + \sum_{t=1}^{T_i}\left\|\hat{s}_t^{(i)}-\tilde{s}_{t,\phi}^{(i)}\right\| \leq \kappa T_i + O(e^{T_i\log(L)}\delta)= O(\kappa T_i + e^{T_i\log(L)}\delta),
    \end{split}
\end{equation*}
where the last inequality follows directly from \cref{thm:code} given that the learned policy is resulted from behavior cloning on the auxiliary trajectories. 
\end{proof}
\vspace{-2mm}

\textbf{Remark.} It may seem that our bound~\eqref{eq:code-bound} is no better than the behavior cloning bound~\eqref{eq:bc_bound} since both of them have a term which grows exponentially as the horizon $T_i$. This, however, is not true as the error bounds $\epsilon$ in~\eqref{eq:bc_bound} and $\delta$ in~\eqref{eq:code-bound} are fundamentally different. 
As is discussed in~\cref{sec:problem}, the error $\epsilon$ captures both optimization error, and the inconsistency between the demonstrations and the policy class. The error caused by this inconsistency is \textit{intrinsic to demonstrations and the policy class} and often cannot be reduced. 
The error $\delta$, however, can be changed by the hyper-parameter $\lambda$. Provided that the auxiliary trajectories are expressive enough\footnote{The expressivity of the auxiliary trajectories is in general not a concern because \emph{1)} they only need to model a relative simple function (compared to the policy): a direct mapping from time to state, and \emph{2)} since the auxiliary trajectories are only used during training, we can overparameterize them without worrying about overfitting.}, we can drive $\delta$ to $0$ by choosing $\lambda\to\infty$. This would restrict the auxiliary trajectories to only those that can be generated by policies in the policy class. Therefore, our approach can be considered as effectively reducing the one-step prediction error of behavior cloning (which aggregated quickly as time increases) at a small cost of imitating auxiliary trajectories that are $\kappa$-close to the original demonstrations. For example, if we are learning to imitate a non-smooth demonstration using a smooth policy class (as is shown in~\cref{fig:CoDE_illustration}), we can imagine the auxiliary trajectories converge to a smoothed version of the demonstrations, still with small $\kappa$, that are likely to produce significantly smaller behavior cloning error $\delta$. 


\vspace{-4mm}
\section{Realization of CoDE for Acceleration-based Systems}\label{sec:instantiation}
\vspace{-3mm}

In this section, we present a realization of CoDE for discrete-time acceleration-based systems, which are commonly encountered in robotics~\cite{Cheng18rmpflow,Ratliff18riemannian}. The system is define as:
\begin{equation}\label{eq:double-integrator}\small
    s_{t+1} = \begin{bmatrix}
    \q_{t+1}\\\qd_{t+1}
    \end{bmatrix} = \begin{bmatrix}
    \q_t + \qd_tT_s\\
    \qd_t + \pi_\theta(\q_t,\qd_t)T_s
    \end{bmatrix}:=f(s_t, \pi_\theta(s_t)),
\end{equation}
where the state $s_t$ consists of a tuple of position $\q_t\in\R^d$ and velocity $\qd_t\in\R^d$, the policy $\pi_\theta(\q_t,\qd_t)$ produces the acceleration action $a_t=\ab_t$, and $T_s$ is the sampling time of the discrete-time system.  
For fully actuated torque-controlled robotics problems, we can often obtain an acceleration-based system through feedback linearization with an inverse dynamics model~\cite{siciliano2010robotics}. We follow the robotics terminology and call $\qd$ the \emph{generalized coordinate} in the \emph{configuration space} (also denoted as $\CC$ space), and $\qd$ the \emph{generalized velocity}. Examples of generalized coordinate include joint angles for a robot manipulator, and the $2$-d position for a planar particle. 

\vspace{-3mm}
\subsection{CoDE for Acceleration-based Systems}
\vspace{-2mm}
Assume that the policy $\pi_\theta$ is Lipschitz continuous, which can be ensured by, e.g., bounding the weights in a neural network. The closed loop system under the acceleration-based dynamics~\eqref{eq:double-integrator} is also Lipschitz continuous (see~\cref{sec:proofs} for proof of the proposition). 
\begin{restatable}{proposition}{lipschitz}\label{prop:lipschitz}
Suppose that the policy $\pi_\theta$ is $\beta$-Lipschitz continuous with respect to $s_t$. Then the mapping $f(\cdot,\pi_\theta(\cdot))$ with $f$ defined in~\eqref{eq:double-integrator} is $L$-Lipschitz continous with respect to the $\ell_2$ norm with $L=\sqrt{2(1+(1+\beta^2)T_s^2)}$.
\end{restatable}

Given~\cref{prop:lipschitz}, we have that the mapping $f(\cdot,\pi_\theta(\cdot))$ is $L$-Lipschitz with $L>1$ in general (unless the policy $\pi_\theta$ admits additional stability properties). Therefore, \cref{thm:bc} and \cref{thm:code} both hold and CoDE can be applied to mitigate covariate shift caused by the inconsistency between the demonstrations and the policy class. 

Further, we note that, for the acceleration-based system~\eqref{eq:double-integrator}, the one-step prediction error is proportional to the action error. Namely, let $s_{t+1}$ and $s_{t+1}'$ be the states after executing actions (accelerations) $\ab_t$ and $\ab_t'$ at state $s_t=(\q_t,\qd_t)$, respectively. Then we have, 
\begin{equation*}\small
    \begin{split}
        \|s_{t+1} - s'_{t+1}\|_2^2 &= \|(\q_t + \qd_tT_s) - (\q_t + \qd_tT_s)\|_2^2+ \|(\qd_t +  \ab_tT_s) - \qd_t + \ab_t'T_s)\|_2^2 = T_s^2\|\ab_t-\ab_t'\|_2^2.
    \end{split}
\end{equation*}
This means that if we can extract the actions from the auxiliary trajectories (through, e.g. differentiation), $\{\tilde{a}_{t,\phi}^{(i)}\}$, we can equivalently put the loss on the action deviation between the learned policy and the auxiliary trajectories:
\begin{equation}\label{eq:collocation-double-integrator}\small
    \begin{split}
        \min_{\theta,\phi} \; &L(\theta,\phi;\nu):=
        \sum_{i=1}^N\sum_{t=0}^{T_i}\left\|s_t^{(i)} - \tilde{s}_{t,\phi}^{(i)}\right\|_2^2+ \nu\,
    \sum_{i=1}^N\sum_{t=0}^{T_i-1}\left\|\tilde{a}_{t,\phi}^{(i)}-\pi_\theta\big(\tilde{s}_{t,\phi}^{(i)}\big)\right\|_2^2, 
    \end{split}
\end{equation}
with the multiplier $\nu >0$, where $\nu=\lambda T_s^2$ with the multiplier $\lambda$ defined in~\eqref{eq:collocation-lagrangian}. We choose loss function $\ell(s,s')=\|s-s'\|_2^2$ in consideration of~\cref{thm:code} and~\cref{prop:lipschitz}. 

\begin{figure*}[!t]
    \vspace{-4mm}
    \centering
    \begin{subfigure}[b]{0.55\textwidth}
        \centering
    	\includegraphics[width=\textwidth]{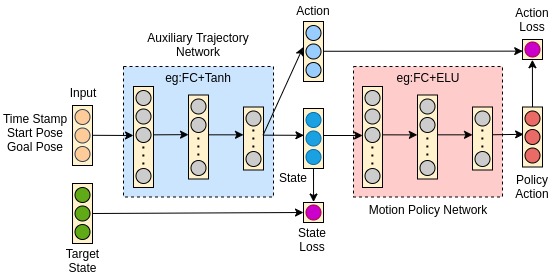}
    \end{subfigure}
	\caption{The architecture for CoDE on acceleration-based systems. See~\cref{sec:architecture} for details. }
	\vspace{-6mm}
	\label{fig:franka_demo}
\end{figure*}

\vspace{-3mm}
\subsection{Auxiliary Trajectory Parameterization}\label{sec:trajectory_net}
\vspace{-2mm}
We now discuss our choice of parameterizing the auxiliary trajectories and how do we extract action sequences $\{\tilde{a}_{t,\phi}^{(i)}\}$. 
For the acceleration-based system~\eqref{eq:double-integrator}, we choose to generate the auxiliary (position and velocity) trajectories from parameterized continuous-time position trajectories $\{{\rho}_\phi^{(i)}\}_{i=1}^N$, where each trajectory is a mapping from (continuous) time $\tau$ to positions $\q(\tau)$, i.e., ${\rho}_\phi^{(i)}:\tau\mapsto\q(\tau)$. We can sample positions $\{\tilde{\q}_{t,\phi}^{ (i)}\}$ from the position trajectories, and obtain the velocities $\{\tilde{\qd}_{t,\phi}^{(i)}\}$ through finite difference: 
    $\tilde{\q}_{t,\phi}^{(i)} = \rho_\phi^{(i)}(t\,T_s),\; \tilde{\qd}_{t,\phi}^{(i)} = \frac{1}{2 \Delta}(\rho_\phi^{(i)}(t\,T_s + \Delta) - \rho_\phi^{(i)}(t\,T_s - \Delta)),$
 where $\Delta$ can be made small to improve numerical accuracy\footnote{$\Delta$ can be set very small to improve the numerical estimation of its gradient. Other numerical methods like the biocomplex-step and hyper-dual methods offer derivative estimation accuracy to machine precision~\cite{agamawi2020comparison}. Also note that finite differencing is done for the auxiliary trajectories rather than the demonstrations. 
}. 
We can similarly calculate the acceleration actions $\{\tilde{a}_{t,\phi}^{(i)}\}$ through finite differencing $\{\qd_t^{(i)}\}$. 
Further, inspired by splines~\cite{epperson2013introduction}, we choose to parameterize the auxiliary trajectories such that their initial and the final states match the demonstrations, i.e., $\tilde{s}_{0,\phi}^{(i)}=s_0^{(i)}$ and $\tilde{s}_{T_i,\phi}^{(i)}=s_{T_i}^{(i)}$ for all $i\in\{1,\ldots,N\}$. Our parameterization is detailed in~\cref{sec:experiment_details}.

\vspace{-3mm}
\subsection{Overall Architecture}\label{sec:architecture}
\vspace{-2mm}
A diagram of the overall architecture for CoDE is shown in~\cref{fig:franka_demo}, in which we have two separate neural networks, one for the auxiliary trajectory, and the other one for the policy. The auxiliary trajectory network takes a tuple of time, start and goal end-effector poses, and obstacle information, if applicable, as input, and outputs a state vector which consists of position and velocity (see~\eqref{eq:traj_net}). An acceleration action is obtained through finite differencing the velocity. The policy network takes the state from the trajectory network and outputs an action. 

To ensure that auxiliary trajectory matches the trajectory demonstration, we have a loss~\eqref{eq:collocation-double-integrator} consists of two parts: \textit{1)} a loss between the target state from the trajectory demonstrations and the auxiliary trajectories, and \textit{2)} a loss that captures the difference between the action generated by the policy and the action output from the auxiliary trajectory network. As is discussed in~\cref{sec:CoDE}, the second loss is imposed such that the auxiliary trajectory becomes an approximated rollout trajectory for the learned policy. By optimizing both losses simultaneously, we can learn 
a policy whose rollout trajectories match both the auxiliary trajectories and the demonstrations.

\vspace{-4mm}
\section{Experiments}\label{sec:experiments}
\vspace{-3mm}

We present the results 
for a simulated 7-degrees-of-freedom (DoF) Franka robotic manipulator on two different tasks: \textit{1)} end-effector lifting and random goal reaching (see \cref{fig:exp_task1}a); and \textit{2)} random goal reaching around a randomly oriented object (see \cref{fig:exp_task2}a). 
These tasks 
contain non-trivial, nonlinear motion that is not reproducible by typical straight-line operational space or Cartesian-space controllers. We include additional details about the experiment setup and data collection in~\cref{sec:experiment_details}. 

\textbf{Policy network:}
We use a 3-layer fully connected neural network\footnote{We also test the proposed algorithm on a structured policy class called Riemannian Motion Policies (RMPs)~\cite{Cheng18rmpflow,li2021rmp2} and observe similar results. We report these results in~\cref{sec:rmp}.} to parameterize the policy. 
For Task 1, the neural network policy takes as input the configuration space state ($14$-dimensional vector with joint angles and velocities) and the $3$-dimensional goal position, and outputs a $7$-dimensional joint acceleration vector as the action. For Task 2, the neural network policy takes in a 4-dimensional cylinder orientation in addition to the robot state vector and the goal position.

\textbf{Auxiliary trajectory network:} We use a set of 3-layer neural networks with $\tanh$ activation function to parameterize the auxiliary trajectories. The details of our auxiliary trajectory parameterization is included in~\cref{sec:experiment_details}. 
For Task 1, we use the $7$-dimensional goal position and orientation (in quaternions) as the external features, and for Task 2, we add the 4-dimensional cylinder orientation as external features in addition to the 7-dimensional goal position and orientation. 

\textbf{Baselines:}
We present two baselines: \textit{1)} behavior cloning (\textsc{BC}) and \textit{2)} behavior cloning with noise injection (\textsc{BC+Noise}), a common practical trick to improve BC performance. 
For \textsc{BC}, we learn the policy by optimizing a loss between the acceleration from the data and the acceleration generated by the policy (equivalent to~\eqref{eq:bc_states} for acceleration-based systems). For \textsc{BC+Noise}, we follow~\citet{ke2020grasping}, and add Gaussian noise with $\sigma = 0.05$ to 20\% of training trajectories.

\begin{figure*}[!t]
    \centering
    \vspace{-4mm}
    \begin{subfigure}[b]{0.3\textwidth}
        \centering
        \small Task 1\\
        \vspace{2mm}
    	\includegraphics[width=0.8\textwidth]{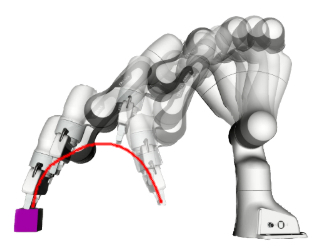}
    	\caption{}
    	\vspace{-2mm}
    \end{subfigure}
    \hfill
    \begin{subfigure}[b]{0.3\textwidth}
        \centering
    	\includegraphics[width=\textwidth]{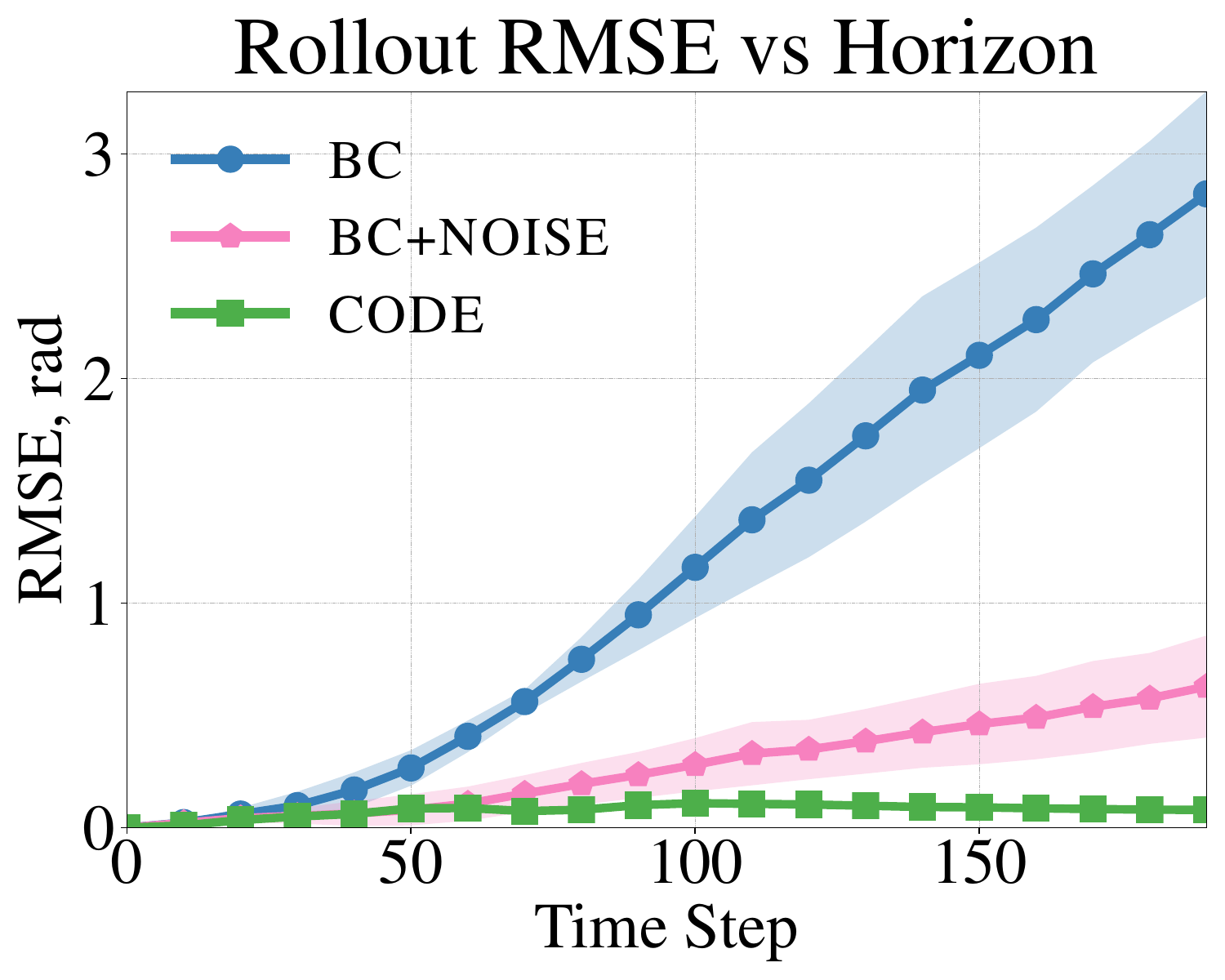}
    	\vspace{-6mm}
        \caption{}
        \vspace{-2mm}
    \end{subfigure}
    \hfill
    \begin{subfigure}[b]{0.3\textwidth}
        \centering
    	\includegraphics[width=\textwidth]{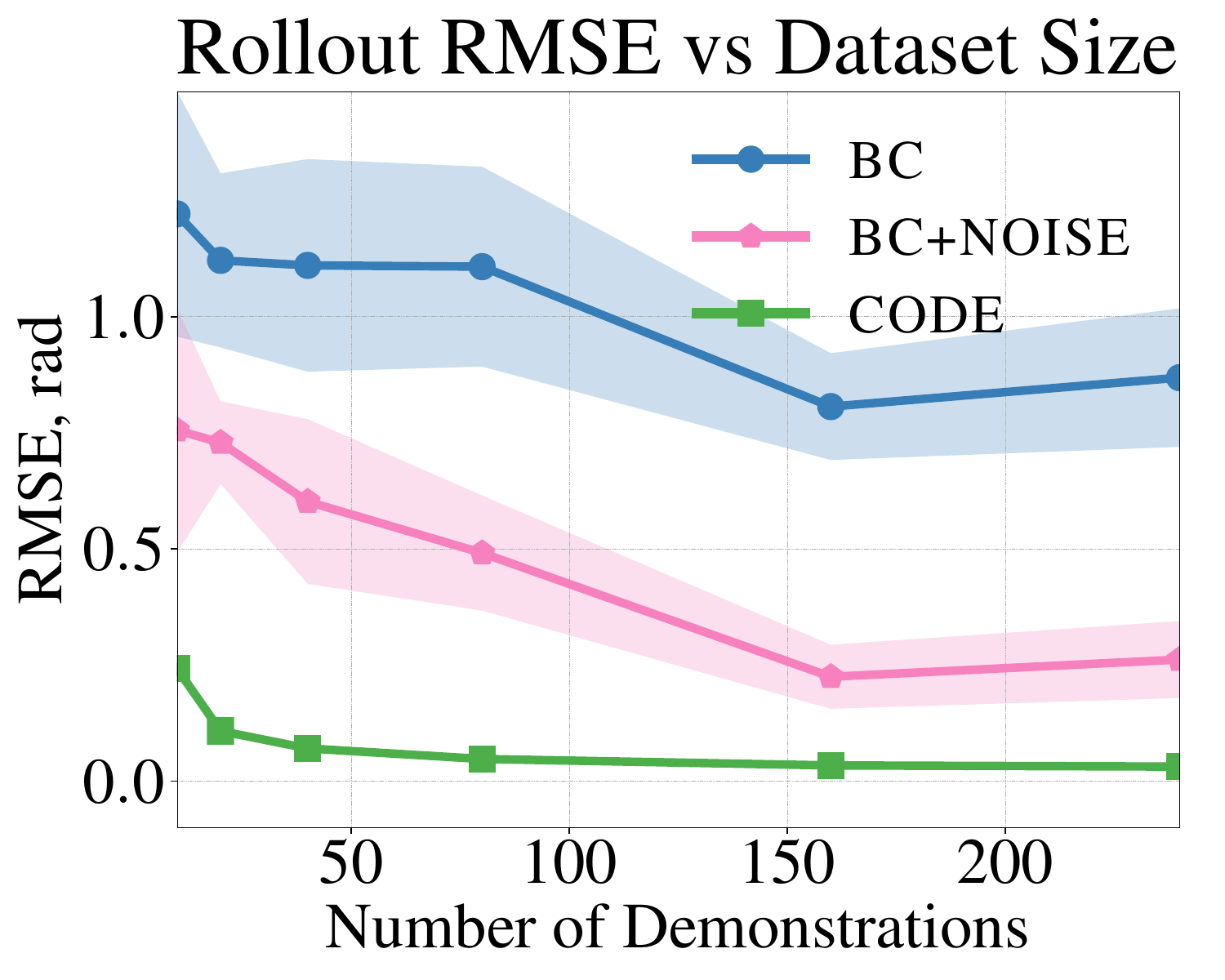}
    	\vspace{-6mm}
        \caption{}
        \vspace{-2mm}
    \end{subfigure}
	\caption{End-effector lifting and random goal reaching (Task 1). 
	(a) Visualization of one demonstration trajectory. 
	The arm’s transparency ranges from light at the beginning of the trajectory to dark at the end. (b) State deviation measured in $\ell_2$ norm for each time step during rollout initialized at the starting point of testing trajectories. 
	(c) Rollout RMSE on the reserved testing trajectories with policies trained on $10$ to $240$ trajectories. 
	For (b) and (c), the solid lines correspond to the mean performance evaluated across 6 random training seeds with policies assessed on $50$ testing trajectories. The shaded area is given by mean $\pm $ standard deviation over the random network initializations. 
	}
	\label{fig:exp_task1}
\end{figure*}

\begin{figure*}[!t]
    \centering
    \vspace{-2mm}
    \begin{subfigure}[b]{0.3\textwidth}
        \centering
        \small Task 2\\
    	\includegraphics[width=0.7\textwidth]{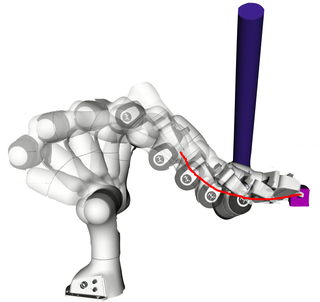}
    	\caption{}
    	\vspace{-2mm}
    \end{subfigure}
    \hfill
    \begin{subfigure}[b]{0.3\textwidth}
        \centering
    	\includegraphics[width=\textwidth]{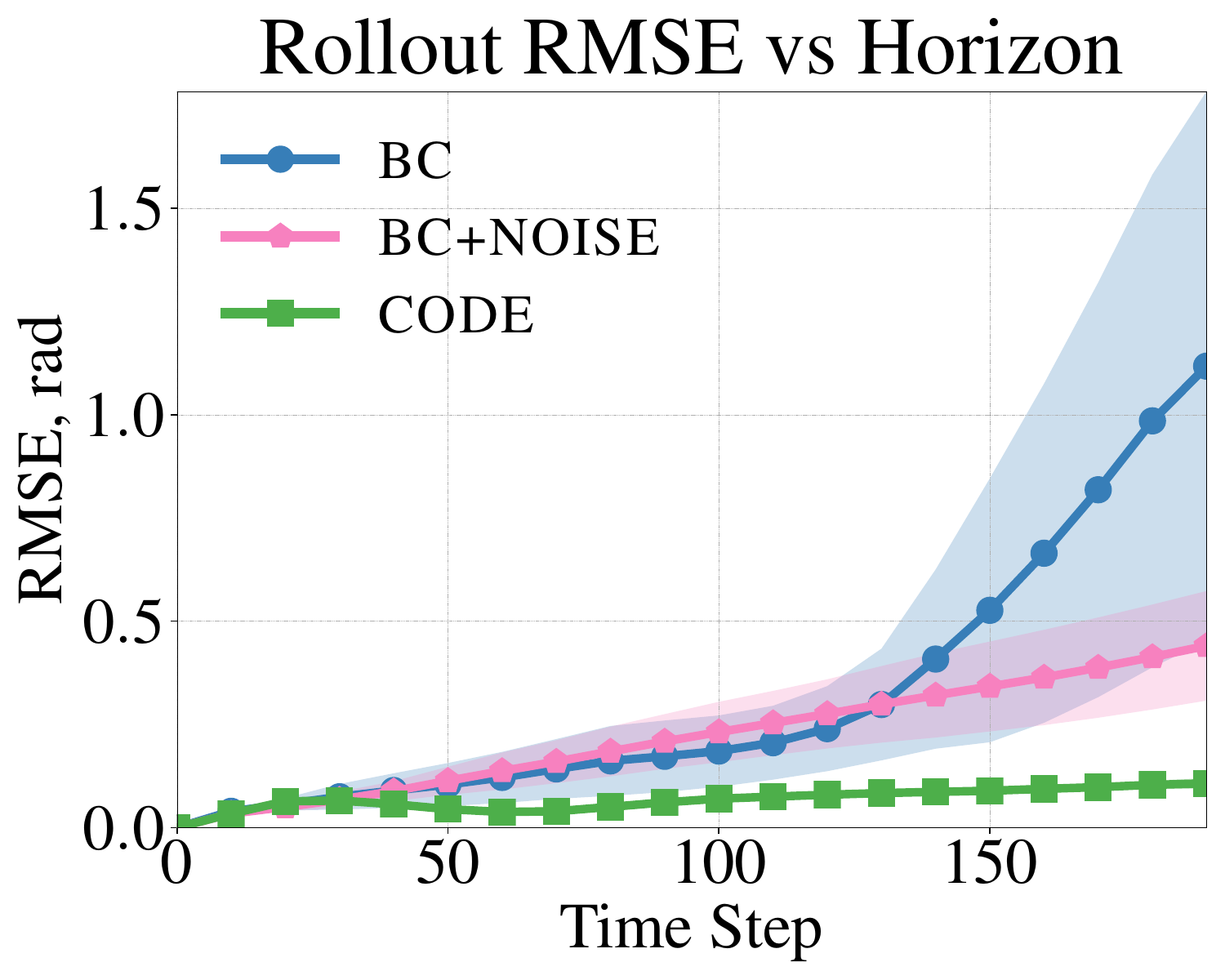}
    	\vspace{-6mm}
    	\caption{}
    	\vspace{-2mm}
    \end{subfigure}
    \hfill
    \begin{subfigure}[b]{0.3\textwidth}
        \centering
    	\includegraphics[width=\textwidth]{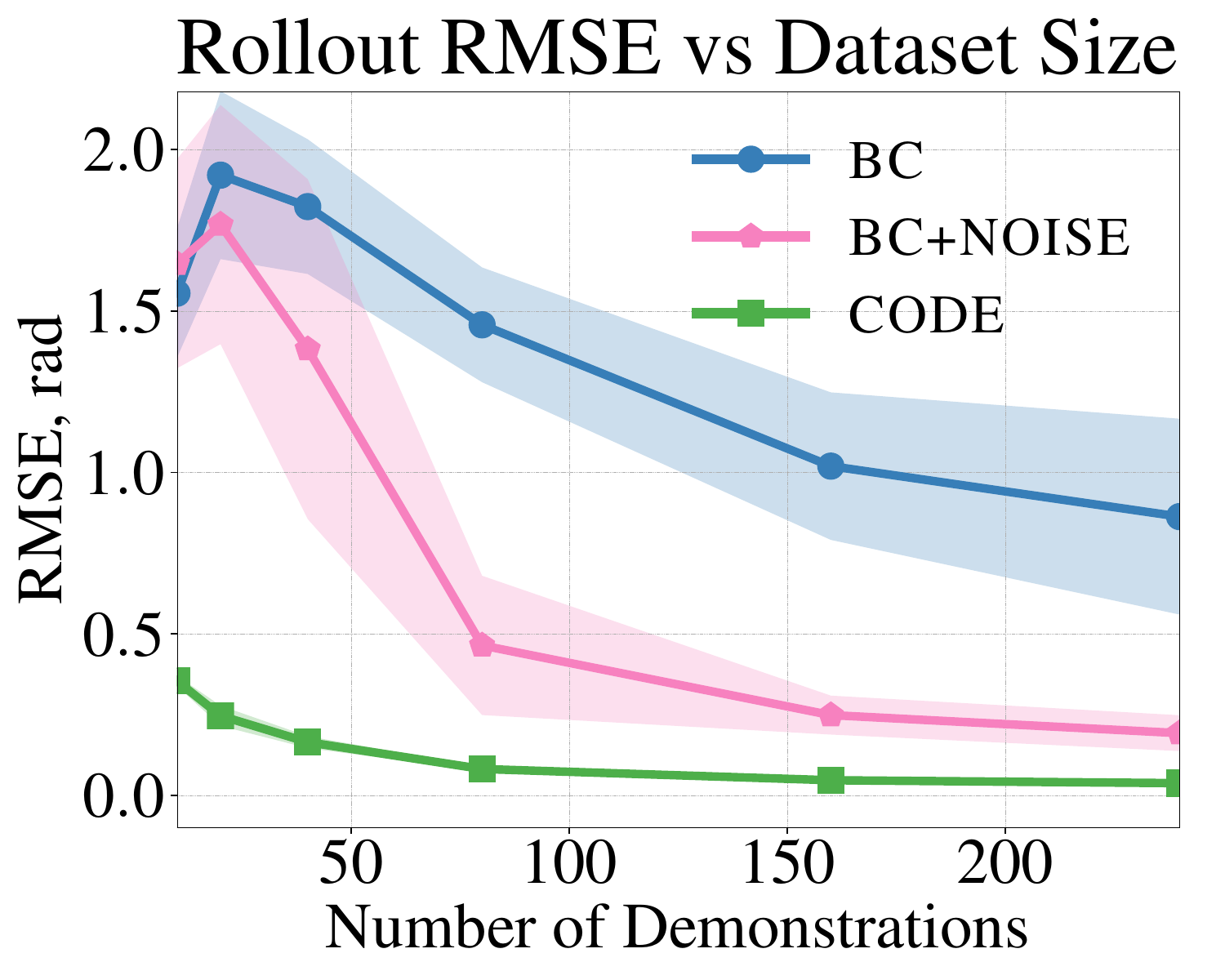}
    	\vspace{-6mm}
    	\caption{}
    	\vspace{-2mm}
    \end{subfigure}
	\caption{
	Random goal reaching around a randomly oriented object (Task 2). (a) Visualization of one demonstration trajectory. (b) State deviation measured in $\ell_2$ norm for each time step during rollout initialized at the starting point of testing trajectories. (c) Rollout RMSE on the testing trajectories with policies trained on $10$ to $240$ trajectories. Same trends are observed as for Task 1~(c.f.~\cref{fig:exp_task1})}
	\vspace{-4mm}
	\label{fig:exp_task2}
\end{figure*}

\textbf{Procedure:} We train \textsc{BC}, \textsc{BC+Noise}, and \textsc{CoDE} on sets of 
trajectories with sizes $[13, 25, 50, 100, 200, 300]$ for each task. Within the trajectories, $80\%$ of them are used for updating model parameters (training set), and $20\%$ are used to assess model performance (validation set). 
For every $10$ epochs, we compute the rollout RMSE with respect to the trajectories within the validation set, and select the best-performing policy based on this. 
We test the selected policy for each task on a fixed set of $50$ trajectories (testing set) that are not presented in the training and validation set. We repeat the process 
across each algorithm, task, and training data size with $6$ different random seeds.

\textbf{Results:} From our experiments, we empirically examine \textsc{CoDE} along two primary performance indices: 
\emph{1)} error accrual rates along multi-step rollouts; and \emph{2)} overall task performance based on complete rollout error assessed on unseen trajectories. Performance mean and variation across all random seeds is captured in \cref{fig:exp_task1} and \cref{fig:exp_task2}.


\textit{Multi-step error accrual.}   We train \textsc{BC}, \textsc{BC+Noise}, and \textsc{CoDE} on specifically $240$ trajectories for each task. To validate our theoretical analysis of error accrual bounds (\cref{thm:bc} and \cref{thm:code}), we measure policy error across rollout horizon lengths. Specifically, we calculate the state deviation measured in $\ell_2$ norm for each time step when the learned policy is initialized at the testing trajectories. For each time step $t$ in the $x$-axis, the $y$-axis represents the statistics of state deviation between the trajectory demonstration and the learned policy rollout. 
We observe that for both Task 1 (\cref{fig:exp_task1}b) and Task 2 (\cref{fig:exp_task2}b), \textsc{BC} generates a multi-step rollout error that grows rapidly with horizon length while \textsc{BC+Noise} error is more attenuated, and \textsc{CoDE} yields the smallest error growth which is nearly flat even for long horizons. Moreover, the performance variation across random initializations is also much larger for \textsc{BC} and \textsc{BC+Noise} than for \textsc{CoDE}. This improved training precision is desirable quality for reducing lengthy training cycles. These results empirically validate our theoretical expectation (\cref{thm:code}) which is that a policy trained with \textsc{CoDE} should exhibit smaller rollout error along horizon length in comparison to a policy trained with \textsc{BC} (\cref{thm:bc}). 


\textit{Task performance.} We test the 
learned policies given by training sets of sizes 13 to 300 (20\% of the trajectories are reserved for the validation set) on a fixed testing set of 50 trajectories. 
As shown in \cref{fig:exp_task1}c and \cref{fig:exp_task2}c, policies trained with \textsc{CoDE} are significantly more accurate and less sensitive to random seeds, given the smaller RMSE means and standard deviations across all testing trials. 
For Task 1, the performance of policies produced by \textsc{CoDE} at only 10 training trajectories match those policies produced by \textsc{BC+Noise} at over 150 training trajectories. Similarly, for Task 2, the performance of policies produced by \textsc{CoDE} at about 40 training trajectories match those policies produced by \textsc{BC+Noise} at over 200 training trajectories. Moreover, policies produced by either BC variant cannot meet the performance of policies produced by CoDE even when trained across the maximum number of training demonstrations. Clearly, \textsc{CoDE} is not only capable of producing the best-performing policies when compared to both BC variants, but also obtains such performance at an order of magnitude lower number of training samples. This quality is quite significant for imitation learning because the production of demonstration data 
can be expensive and time consuming. 

For policies trained with 240 demonstrations, we evaluate the mean target errors (distance between the end-effoctor position at the end of the trajectory and the target position) for both tasks and 
the mean collision rates for Task 2 over 6 random seeds on the testing set.  
For Task 1, the mean target errors for \textsc{BC}, \textsc{BC+Noise} and \textsc{CoDE} are $0.89m$, 
$0.41m$, and $0.01m$, respectively. 
For Task 2, the mean target errors for \textsc{BC}, \textsc{BC+Noise} and \textsc{CoDE} are $0.78m$, $0.41m$, and $0.02m$, respectively; 
and the mean collision rates 
are $49.0\%$, $38.3\%$, and $5.3\%$, respectively. 
The policies learned with \textsc{CoDE} are significantly 
safer and more effective.


\textbf{Remark on baseline performance.} The poor BC performance is mainly caused by insufficient data coverage in the {velocity space}. The demonstrated trajectories all have relatively low velocity with a maximum of 1.14 $rad/s$
and the median of 0.48 $rad/s$ (in $\ell_2$-norm). 
As the number of demonstrations increases, the data coverage in the joint angle space improves, yet velocity coverage remains poor. This issue can be alleviated through noise injection, which provides slightly better coverage. 
This is demonstrated by the improved performance of \textsc{BC+Noise} over \textsc{BC} for both tasks (See~\cref{fig:exp_task1,fig:exp_task2}). However, \textsc{CoDE} yields significantly better results compared to both BC variants, as \textsc{CoDE} directly deals with the inconsistency between the demonstrations and the learner. 

\vspace{-4mm}
\section{Discussions}\label{sec:discussions}
\vspace{-3mm}
\textbf{Limitations \& future work. } As is discussed in~\cref{sec:problem}, CoDE assumes the availability of a dynamics model and state demonstrations, which is true for our experimental setup. In the cases when, e.g., only visual demonstrations are available, one can pre-train a latent state encoding and a dynamics model, e.g. as is done in~\cite{byravan2018se3,das2020model}. Although our analysis can be generalized to approximate dynamics model (see~\cref{fnote:approx}), we leave the experimental validation of this scenario for future work. In future work, we can also extend CoDE to learning dynamics models and polices jointly. 

\textbf{Potential negative social impacts.}
CoDE is a general imitation learning method that operates on demonstration data. In practice, the demonstration data may contain biases which can be learned by the resulting policies. Increased care and effort in detecting and removing potentially harmful data biases is generally advised.

\newpage
\bibliographystyle{unsrtnat}
\bibliography{references.bib}


\newpage
\appendix
\section{Proof of~\cref{prop:lipschitz}}\label{sec:proofs}

\lipschitz*

\begin{proof}
Consider two states, $s_t:=(\q_t,\qd_t)$, and $s'_t:=(\q'_t,\qd'_t)$, where we abuse our notation to use $(\cdot,\cdot)$ to denote vector concatenation. Then, we have,
\begin{equation*}
    \begin{split}
        &\|s_{t+1} - s'_{t+1}\|_2^2\\
        =& \|\q_{t+1} - \q'_{t+1}\|_2^2 + \|\qd_{t+1} - \qd'_{t+1}\|_2^2\\
        =& \|(\q_t + \qd_tT_s) - (\q'_t + \qd'_tT_s)\|_2^2 + \|(\qd_t + \pi_\theta(s_t)T_s) - (\qd'_t + \pi_\theta(s_t')T_s)\|_2^2\\
        =& \|(\q_t - \q_t') + T_s (\qd_t -\qd'_t)\|_2^2 + \|(\qd_t - \qd'_t) + T_s( \pi_\theta(s_t) - \pi_\theta(s_t'))\|_2^2\\
        \leq& 2\|\q_t - \q'_t\|_2^2 + 2T_s^2 \|\qd_t - \qd'_t\|_2^2 + 2\|\qd_t - \qd'_t\|_2^2 + 2T_s^2 \|\pi_\theta(s_t) - \pi_\theta(s_t')\|_2^2
    \end{split}
\end{equation*}
where the last inequality follows from the the parallelogram law:  $2\|a\|_2^2 + 2\|b\|_2^2 = \|a+b\|_2^2 + \|a-b\|_2^2\geq \|a+b\|_2^2$. By Lipschitz continuity of the policy $\pi_\theta$,
\begin{equation*}
    \begin{split}
        \|s_{t+1} - s'_{t+1}\|_2^2\leq& 2(1+T_s^2)(\|\q_t - \q'_t\|_2^2 + \|\qd_t - \qd'_t\|_2^2) + 2\beta^2 T_s^2 \|s_t-s_t'\|_2^2\\
        =& 2 (1 + (1 + \beta^2)T_s^2) \|s_t - s_t'\|_2^2,\textbf{}
    \end{split}
\end{equation*}
The proposition then follows from taking the square root on both sides of the equation. 
\end{proof}

\section{Experiment Details}\label{sec:experiment_details}
The first task is to reach a randomly generated end-effector pose from a random initial configuration. The end-effector poses are sampled in the region close to the robot, centered at $(0.5, 0, 0)$ in the base coordinates frame of the robot with a standard deviation of $0.2m$. The second task is to reach randomly sampled targets around a randomly oriented cylinder located in front of the robot from a fixed initial configuration, where the center of the cylinder is located at $(0.5, 0, 0.6)$ in the robot base coordinates frame, and the pitch angle of the cylinder is sampled in the range of $(-\pi/3, \pi/3)$.

\textbf{Dataset:}
Although our approach is applicable of learning from any trajectory data (include but not limited to trajectories generated from reactive policies, motion planners, and human demonstrations, etc.), we choose to collect trajectories from a sophisticated reactive policy so that we can perform controlled quantitative evaluation. 
We collected 530 demonstration trajectories (from which we extract smaller subsets in our experiments to study data efficiency) in the configuration space of the Franka robot using a state machine-driven Riemannian Motion Policy (RMP) system\footnote{The use of a state machine in the demonstrator and a more general deep RMP parameterization in the policy enforces that the demonstrations lie outside the learner's policy class.
}~\cite{Cheng18rmpflow} in simulation. \cref{fig:exp_task1}a visualizes one of the demonstrations; this behavior is a common primitive for tabletop manipulation. More details on our data collection setup is included in~\cref{sec:data_collection} for completeness. 
The data generated consists of robot lifting, reaching, and obstacle avoidance motions and therefore does not contain any personally identifiable information or offensive content. 

Each demonstration data point includes the robot's 7-DoF joint angles, velocities, and accelerations, along with the corresponding time stamp. For each demonstration, the initial end-effector pose and the goal pose are also recorded. The sample time for each data point is $0.01$ second, and the time horizon for each demonstration ranges from 2 to 6 seconds. We want to learn a motion policy that generates acceleration in the $\CC$ space at a given $\CC$ space state (position and velocity). We reserved 50 demonstration trajectories for testing, and make the rest available for training, where 80 $\%$ are used for training, and 20 $\%$ are used for validation.

\textbf{Auxiliary trajectory parameterization. }
We achieve this through reparameterization: 
We first construct a function $\psi_\phi^{(i)}(\tau)$ directly parameterized by parameter $\phi$ (Note that $\psi_\phi^{(i)}(\tau)$ does not need to match the initial and final states of the demonstration). Then, we obtain the parameterized positional trajectory $\rho_\phi^{(i)}(\tau)$:
\begin{equation}\label{eq:traj_spline}\small
    \begin{split}
        \rho_\phi^{(i)}(\tau) &:= \left(\frac{T_i-\tau}{T_i} + \frac{\tau(T_i-\tau)(T_i-2\tau)}{T_i^3}\right)\q_0^{(i)}+ \left(\frac{\tau}{T_i} - \frac{\tau(T_i-\tau)(T_i-2\tau)}{T_i^3}\right)\q_{T_i}^{(i)}\\
        &\qquad+ \frac{\tau(T_i-\tau)^2}{T_i^2}\qd_0^{(i)} - \frac{\tau^2(T_i-\tau)}{T_i^2}\qd_{T_i}^{(i)}+ \tau^2(T_i-\tau)^2\psi_\phi^{(i)}(\tau),
    \end{split}
\end{equation}
where $\q_0^{(i)}$, $\qd_0^{(i)}$ are the initial position and velocity of the demonstrated trajectory, and $\q_{T_i}^{(i)}$, $\qd_{T_i}^{(i)}$ are the final position and velocity. This formulation~\eqref{eq:traj_spline}, independent of the choice of $\psi_\phi^{(i)}(\tau)$, ensures that the initial and final states of the positional trajectory $\rho_\phi^{(i)}$ matches the demonstration. 
Therefore, one can use \emph{any} parameterization of $\psi_\phi^{(i)}$ that provides sufficient expressivity. 
For example, one can choose to parameterize $\psi_\phi^{(i)}(\tau)$ as a \emph{single} neural network, with
$\psi_\phi^{(i)}(\tau) := \psi(\tau, \q_0^{(i)}, \qd_0^{(i)}, c^{(i)}; \phi)$,
where $c^{(i)}$ is some external features for the $i$th trajectory, e.g., the goal position, obstacle information, etc. Such a shared parameterization can leverage correlations between demonstrations to further simplify the policy learning sub-problem by smoothing noise. In our experiment, we use a similar parameterization with a neural network for each dimension: 
\begin{equation}\small\label{eq:traj_net}
    \psi_\phi^{(i)}(\tau) := \begin{bmatrix}
    \psi_1(\tau, \texttt{fk}(\q_0^{(i)}), c^{(i)}; \phi_1) & 
    \cdots &
    \psi_d(\tau, \texttt{fk}(\q_0^{(i)}), c^{(i)}; \phi_d)
    \end{bmatrix}^\t,
\end{equation}
where $\texttt{fk}$ is the forward kinematic mapping that maps joint angles to end-effector position and orientation (represented by quaternions), and $\phi_j$ is the parameter of the neural network for the $j$th dimension for $j=1,\ldots,d$. For each neural network, we use 256 and 128 hidden units for the first and second hidden layer, respectively.

\textbf{Sensitivity to Policy Capacity}
We train policy networks with hidden layer sizes of $(256, 128, 64)$, $(128, 64, 32)$, and $(64, 32, 16)$ with 100 trajectories for Task $1$ over $6$ random seeds, and the mean testing RMSEs of CoDE are $0.048$, $0.049$, and $0.055$ $rad$, respectively. As the policy capacity decreases, the performance of CoDE decays gracefully.

\textbf{Training details. }
We use a 3-layer fully connected neural network 
with exponential linear units ($\mathrm{elu}$ activation 
to parameterize the policy. 
The number of hidden units for the two hidden layers are $256$, $128$ and $64$, respectively.
We trained BC, BC+Noise and CoDE 
with Adam optimizer~\cite{kingma14arxiv_adam}. We use learning rate of $5\times10^{-3}$ and weight decay rate of $10^{-10}$. The learning rate decreases by a factor of $0.9$ during training if the loss is not decreasing for more than 500 epochs. The maximum number of training epochs is $5\times10^3$, and the training process can be terminated before it reaches the maximum number of epochs if the training loss is not decreasing for over 500 epochs while the learning rate reaches the minimum value $10^{-6}$. We use a batch size of 2000.

\textbf{Compute \& Time.} The total time to generate all reported results was approximately 500 hours on a single desktop machine with $16$ Intel(R) Core(TM) i7-9800X CPU @$3.80$GHz and an NVIDIA GeForce RTX 3090 GPU. This includes training, validation, and testing on six random network initializations for the NN and RMP policy class trained under behavior cloning (BC), behavior cloning with noise injection (BC+Noise), and CoDE, with dataset size ranging in $[13,25,50,100,200,300]$. 
For an NN policy, the training time for BC, BC+Noise, and CoDE for the dataset of size $100$ are $1195s$, $1202s$, and $1258s$, respectively. 
For an RMP policy, the training time for BC, BC+Noise, and CoDE for the dataset of size $100$ are $3150s$, $3175s$, and $4975s$, respectively. 
The training time scales linearly with dataset size. 

\section{Riemannian Motion Policies}\label{sec:rmp}
Here in this appendix, we consider a structured policy called Riemannian Motion Policy (RMP)~\cite{Cheng18rmpflow,Ratliff18riemannian}. We first briefly introduce the RMP framework\footnote{We refer the readers to~\cite{Cheng18rmpflow} for a more formal and comprehensive treatment of the subject. }, then present the experimental results for CoDE learning with the RMP policy class. 

\textbf{Remark: }In this appendix, we try to keep our notations consistent with~\cite{Cheng18rmpflow} and other work in the literature~\cite{li2019multi,Rana20learningRMP,li2021rmp2}. This, however, inevitably causes a few notations to refer to different quantities than in the main text of this manuscript and other appendices. Notably, $\phi$ in this appendix denotes the mapping between configuration space and the task space, while it refers to the parameters of the auxiliary trajectories for the rest of the paper. 

\subsection{Motion Generation for Robotic Systems}

Consider a robot with a $d$-dimensional \textit{configuration space} $\CC$ with \textit{generalized coordinate}\footnote{An example of a generalized coordinate is the joint angles for a $d$-dof robot manipulator. } $\q\in\R^d$. The time-derivative of the generalized coordinate, $\qd$, is commonly referred to as the \emph{generalized velocity}. For the sake of terminology simplicity, we slightly abuse the terminology and call $\q$ and $\qd$ the position and velocity as if the configuration space is Euclidean. We further assume that the system is feedback linearized in such a way that we can control the acceleration $\qdd$ on the configuration space $\CC$, e.g. the joint angular accelerations on a robot manipulator. As is discussed in~\cref{sec:instantiation}, such assumption holds in most torque-driven fully-actuated robotic systems~\cite{siciliano2010robotics}, such as robot manipulators and holonomic mobile robots. 

The task that a robot needs to achieve, on the other hand, is often more convenient to be described on a different space, commonly called the task space $\TT$. For example, goal-reaching can be defined in a 3-d Euclidean space describing the end-effector's position in relation to the goal, and obstacle-avoidance can be treated in the 1-d distance spaces between points along the robot and obstacles in the environment. Complex tasks, e.g., goal reaching while avoiding joint limits and obstacles, usually require the robot to achieve desired behaviors in multiple spaces, each corresponding to a \textit{subtask}, such as goal-reaching, collision avoidance, etc. We called those spaces \emph{subtask spaces}, denoted $\{\TT_k\}_{k=1}^K$, where $K$ is the number of subtasks. Let $\phi_k:\CC\to\TT_k$ be the mapping from the configuration space $\CC$ to subtask space $\TT_k$ (note the overload of $\phi$ in this appendix as discussed in the remark above). For example, when the subtask space $\TT_k$ is the end-effector frame, $\phi_k$ is the forward kinematic mapping. The goal of \textit{motion generation} is to provide a configuration space policy $\qdd =\pi(\q, \qd)$ such that the trajectory in the task space, $\phi_k(\q(t))$, achieves the desired behavior.

\subsection{Riemannian Motion Policies}

The Riemannian Motion Policy (RMP) framework~\cite{Ratliff18riemannian,Cheng18rmpflow} describes each subtask as \emph{1) }a desired behavior in terms of acceleration on the subtask space as well as \emph{2) }a state-dependent importance weight of the subtask policy in relation to other subtasks. The two combined are called a Riemannian Motion Policy (RMP). Concretely, let $\x_k:=\phi_k(\q)\in\R^n$ be the coordinate on the $k$th subtask space, then an RMP $(\ab_k,\M_k)^{\TT_k}$ is a tuple of the acceleration policy $\ab_k:(\x_k,\xd_k)\mapsto \ab_k(\x_k,\xd_k)\in\R^n$ and the importance weight policy (referred to as inertia matrix in~\cite{Cheng18rmpflow}) $\M_k:(\x_k,\xd_k)\mapsto\M_k(\x_k,\xd_k)\in\R^{n\times n}_+$. The state-dependent importance weight allows for more flexibility in shaping the robot behaviors. For example, the importance weight of collision avoidance should be large when the robot is closed to an obstacle and/or is moving fast towards it. Conversely, when the robot is far or is moving away from the obstacle, the weight should be zero or near-zero to avoid interfering with other subtasks.

As is observed by~\cite{li2021rmp2}, the RMP framework solves the following optimization problem to generate the configuration space acceleration policy $\pi(\q,\qd)$:
\begin{equation}\label{eq:rmp-ls}
   \pi(\q,\qd) \; = \; \min_{\ab \in \R^d}  \; \sum_{k=1}^K\,\frac12\,\Big\| \Jb_k \ab + \Jd_k\qd  - \ab_k\Big\|_{\M_k}^2,
\end{equation}
where $\J_k=\frac{\partial\phi_k}{\partial\q}$ is the Jacobian matrix of mapping $\phi_k$ and $\Jd_k$ is the time-derivative of the Jacobian matrix $\J_k$. To understand~\eqref{eq:rmp-ls}, note that
\begin{equation*}
    \frac{d}{dt}\left(\frac{d}{dt} \x_k\right) = \frac{d}{dt}\left(\frac{d}{dt} \phi_k(\q)\right) = \frac{d}{dt}\left(\J_k\qd\right) = \J_k \qdd + \Jd_k\qd.
\end{equation*}
Therefore, the optimization problem~\eqref{eq:rmp-ls} is a least-squares problem defined on the subtask spaces. 

In the literature, there are two efficient strategies for solving the optimization problem~\eqref{eq:rmp-ls}, proposed by  \cite{Cheng18rmpflow} and \cite{li2021rmp2}, respectively. We choose to use the approach by~\cite{li2021rmp2}, which uses automatic-differentiation libraries, e.g., TensorFlow and PyTorch, to compute the closed-form solution to the optimization problem~\eqref{eq:rmp-ls}:
\begin{align} \label{eq:closed-form solution}
    \pi(\q,\qd) = \left( \sum_{k=1}^K \Jb_k^\t \M_k \Jb_k \right)^{\dagger}
    \left( \sum_{k=1}^K \Jb_{k}^\t  \M_k  (\ab_k - \Jd_k\qd)   \right),
\end{align}
where $\dagger$ denotes the Moore–Penrose inverse. This allows us to differentiate through~\eqref{eq:closed-form solution} with automatic-differentiation libraries and learn parameterized RMPs.

\subsection{Parameterization of Riemannian Motion Policies} 

In the experiments (\cref{sec:experiments}), we define the RMP policy class as given by $K=2$ learnable RMPs. The end-effector space RMP encodes the goal-reaching behavior with $\phi_1:\q\mapsto\texttt{fk}(\q)-\x_g$, where \texttt{fk} is the forward kinematics mapping and $\x_g$ is the goal position. The other RMP is defined in the configuration space $\phi_2:\q\mapsto\q$, which serves as a residual policy shaping the configuration space behavior.

We parameterize both RMPs as multi-layer neural networks. For each RMP, we separately parameterize the acceleration policy $\{\ab_k\}_{k=1,2}$ and the importance weight $\{\M_k\}_{k=1,2}$. The acceleration policies are represented by fully-connected neural networks with $\mathrm{elu}$ activation function and two hidden layers of $128$ and $64$ units, respectively. To ensure that the importance weight $\M_k$ is positive definite, we follow the neural network architecture described in~\cite{Rana20learningRMP}: a fully-connected neural network with output dimension $d(d+1)/2$ is used to predict the entries of a lower-triangular matrix $\Lb_k$, which serves as the Cholesky decomposition of the importance weight matrix $\M_k$. Further, a small positive offset of $10^{-5}$ is added to the diagonal entries of $\Lb_k$ to ensure that the diagonal elements are strictly positive. For the experiments, we use $\mathrm{elu}$ activation function with two hidden layers of size $128$ and $64$, respectively. 

\subsection{Experimental Results}\label{sec:rmp_results}
Following the same training procedure in~\cref{sec:experiments}, we train \textsc{BC}, \textsc{BC+Noise}, and \textsc{CoDE} with the structured policy class RMPs on the same dataset for each task, and we report the results evaluated on the same testing set here. We observe that for both Task 1 (\cref{fig:exp_task1_both}) and Task 2 (\cref{fig:exp_task2_both}), the RMP policy generates similar results comparing to the neural network policy in terms of both the \textit{multi-step error accrual} and the \textit{task performance} as described in~\cref{sec:experiments}. One can observe that the RMP policy class offers better error rates than the neural network policy when trained with either the baselines or CoDE. This observation empirically shows the hypothesis that the inductive bias RMPs provides allow for achieving higher performing, more data efficient policies than unstructured policies. For RMP policies, this performance gain is more apparent when trained on fewer demonstrations, which further facilitates data efficiency.
\begin{figure*}[!htb]
    \begin{subfigure}[b]{0.5\textwidth}
        \centering
    	\includegraphics[width=\textwidth]{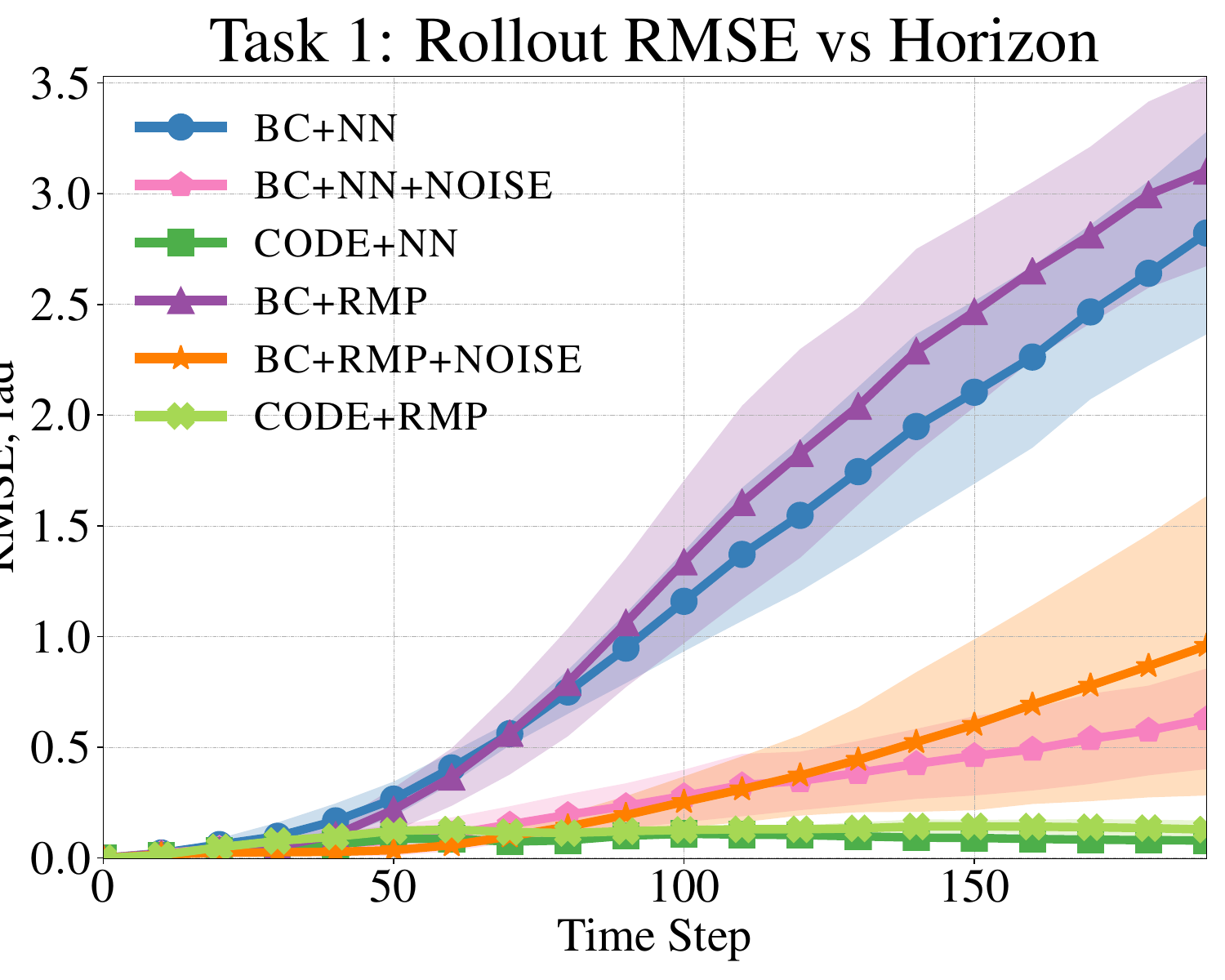}
    	\vspace{-6mm}
        \caption{}
        \vspace{-2mm}
    \end{subfigure}
    \begin{subfigure}[b]{0.5\textwidth}
        \centering
    	\includegraphics[width=\textwidth]{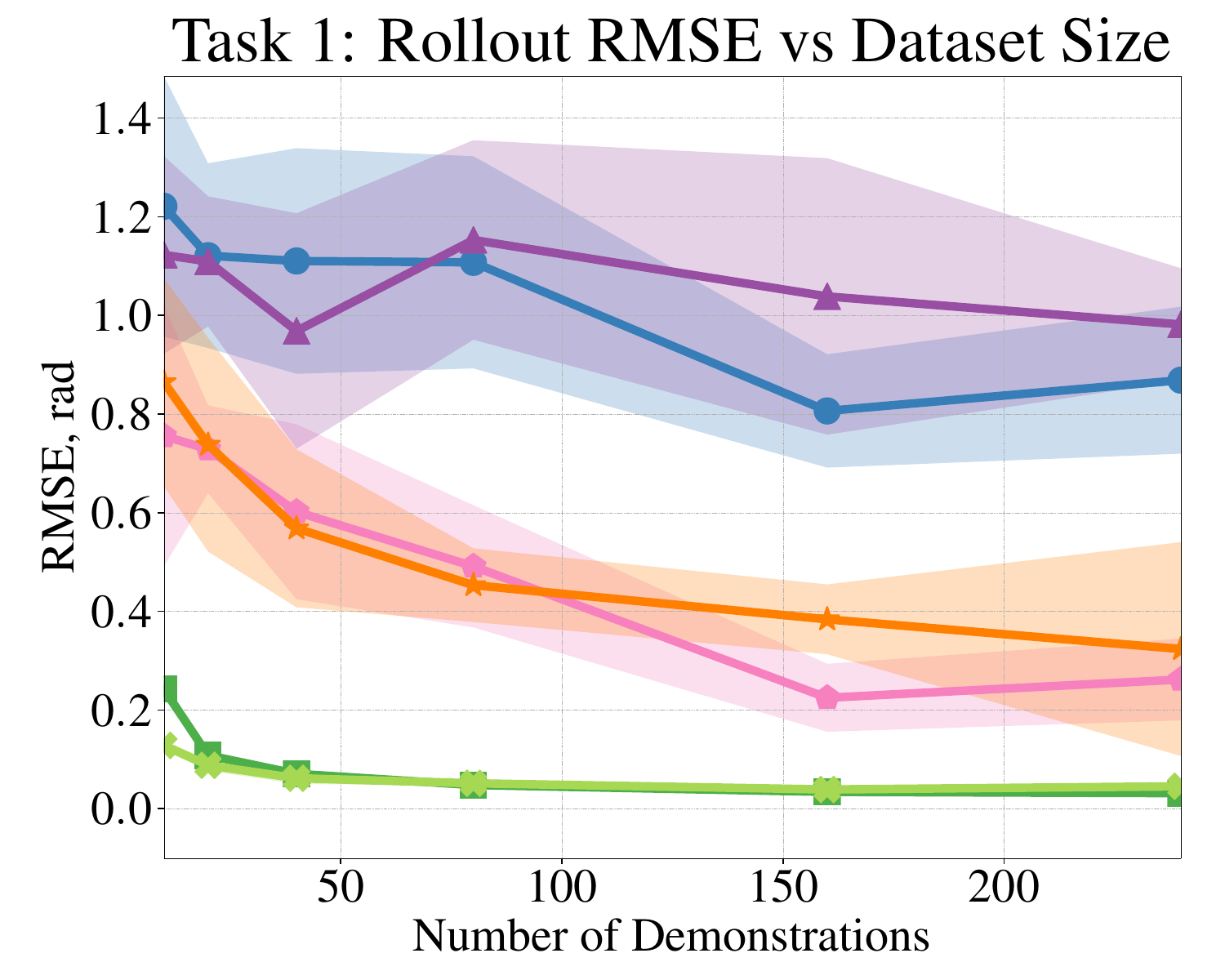}
    	\vspace{-6mm}
        \caption{}
        \vspace{-2mm}
    \end{subfigure}
	\caption{End-effector lifting and random goal reaching (Task 1). 
	(a) State deviation measured in $\ell_2$ norm for each time step during rollout initialized at the starting point of testing trajectories. 
	(b) Rollout RMSE on the reserved testing trajectories with policies trained on $10$ to $240$ trajectories. 
	For (a) and (b), the solid lines correspond to the mean performance evaluated across 6 random training seeds with policies assessed on $50$ testing trajectories. The shaded area is given by mean $\pm $ standard deviation over the random network initializations. 
	}
	\label{fig:exp_task1_both}
\end{figure*}

\begin{figure*}[!htb]
    \begin{subfigure}[b]{0.5\textwidth}
        \centering
    	\includegraphics[width=\textwidth]{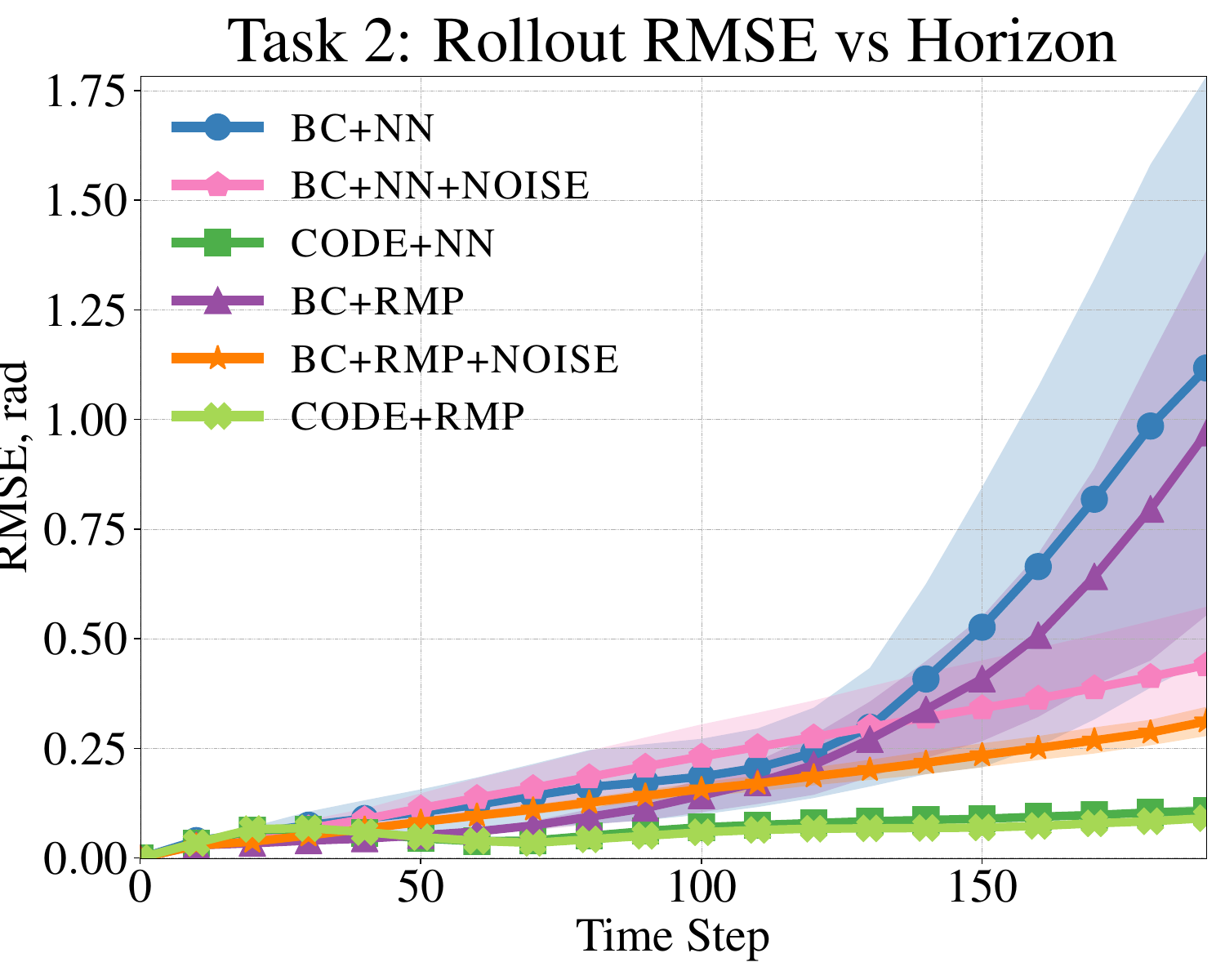}
    	\vspace{-6mm}
    	\caption{}
    	\vspace{-2mm}
    \end{subfigure}
    \begin{subfigure}[b]{0.5\textwidth}
        \centering
    	\includegraphics[width=\textwidth]{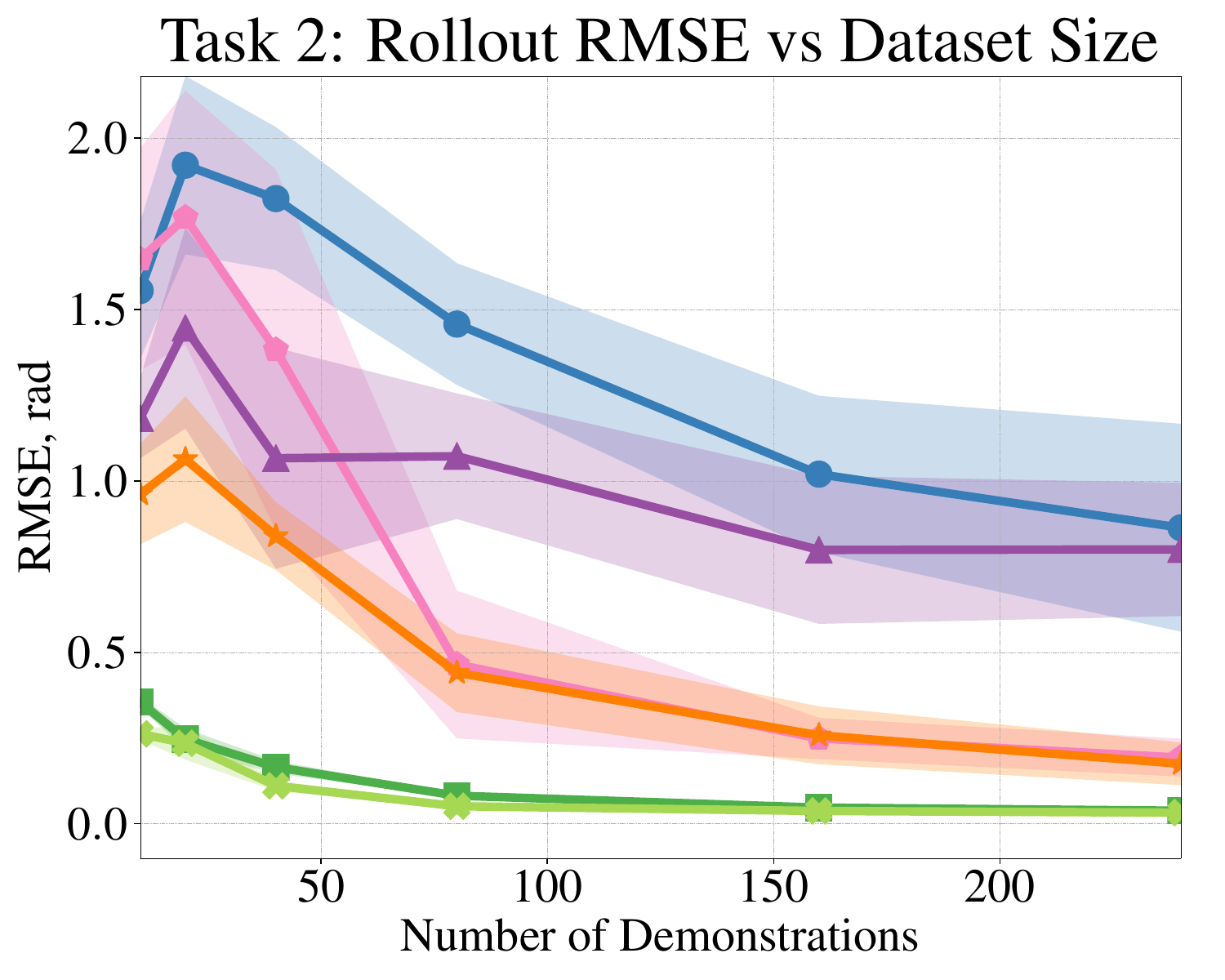}
    	\vspace{-6mm}
    	\caption{}
    	\vspace{-2mm}
    \end{subfigure}
	\caption{
	Random goal reaching around a randomly oriented object (Task 2). (a) State deviation measured in $\ell_2$ norm for each time step during rollout initialized at the starting point of testing trajectories. (b) Rollout RMSE on the testing trajectories with policies trained on $10$ to $240$ trajectories. Same trends are observed as for Task 1~(c.f.~\cref{fig:exp_task1})}
	\vspace{-4mm}
	\label{fig:exp_task2_both}
\end{figure*}

For policies trained with 240 demonstrations, we evaluate the mean target errors (distance between the end-effoctor position at the end of the trajectory and the target position) for both tasks and 
the mean collision rates for Task 2 over 6 random seeds on the testing set. From~\cref{table:target_error_collision_rate}, we observe that the policies learned with \textsc{CoDE} are significantly safer and more effective. Again, one can observe that the RMP policy generates similar results comparing to the neural network policy in terms of both the mean target error and the mean collision rate.

\begin{table}[!htb]
	\caption{Results for Target Errors and Collision Rate.}
	\label{table:target_error_collision_rate}
	\begin{center}
		\begin{tabular}{| c | l | c | c | c | c | c | c |}
			\hline 
			\multicolumn{2}{|c|}{Policy} & \multicolumn{3}{|c|}{NN} & \multicolumn{3}{|c|}{RMP} \\
			\hline
			\multicolumn{2}{|c|}{Algorithm} & BC & BC+NOISE & CODE & BC & BC+NOISE & CODE \\
			\hline
			Task1    
			& Target Error (m)& 0.89 & 0.41 & 0.01 & 0.79 & 0.43 & 0.01 \\
			\hline 
			\multirow{2}*{Task2}     
			& Target Error (m)& 0.78 & 0.41 & 0.02 & 0.66 & 0.25 & 0.01 \\
			& Collision Rate (\%) & 49.0 & 38.3 & 5.3 & 64.3 & 24.5 & 5.0 \\
			\hline 
		\end{tabular}
	\end{center}
\end{table}

\section{Comparison with Independent Parameterization of Auxiliary Trajectories}\label{sec:traj_net_individual}

Traditionally, collocation methods parameterize each auxiliary trajectory independently~\cite{Hargraves87jgcd_collocation,Hereid16icra_collocationHumanoid,Posa16icra_collocationKinematic}, i.e., use a separate set of parameters $\phi^{(i)}$ for each auxiliary trajectory $\{\Tilde{s}_{t,\phi}^{(i)}\}$. We, however, jointly parameterize \emph{all} auxiliary trajectories with one set of parameters $\phi$, and provide the initial configuration $\q_0^{(i)}$ and external features $c^{(i)}$ as input to the neural network so that the network can output $N$ auxiliary trajectories given the initial configuration and external features of the $N$ demonstrations (see \eqref{eq:traj_net} in~\cref{sec:trajectory_net}). 

We choose this joint parameterization for mainly two reasons. First, the joint parameterization is easier to implement for mini-batch training, as there is no need to identify the index of the auxiliary trajectories. Second, the number of parameters in the joint parameterization is independent of the number of demonstrations, where as it is linear for the independent parameterization. In this appendix, we empirically show that, \textit{the joint parameterization has similar performance as the traditional independent parameterization while being more convenient to implement and more parameter efficient}. 

\textbf{Independent Parameterization:} For comparison, we introduce an independent auxiliary trajectory parameterization which uses an individual neural network to represent each auxiliary trajectory:
\begin{equation}\small\label{eq:traj_net_independent}
    \psi_{\phi^{(i)}}^{(i)}(\tau) := \psi(\tau; \phi^{(i)}),
\end{equation}
where $\phi^{(i)}$ is the parameter for the $i$th auxiliary trajectory. Note that since a different set of parameters is used for each trajectory, the external features $c^{(i)}$ and initial configuration $\q_0^{(i)}$ are no longer needed to distinguish between trajectories.  

Similar to~\cref{sec:trajectory_net}, we use a set of 3-layer neural networks with $\tanh$ activation functions to parameterize each auxiliary trajectory. We then obtain the positional trajectories $\{\rho_{\phi^{(i)}}^{(i)}\}_{i=1}^N$ through~\eqref{eq:traj_spline} and sample states and actions from them. For the independent parameterization, we use \textit{a smaller neural network} with $16$ and $8$ units for the first and the second hidden layer, as it only needs to encode a single trajectory. For bookkeeping, from now on, we call the independent parameterization~\eqref{eq:traj_net_independent} as the \textit{multi-neural-net} parameterization and the joint parameterization~\eqref{eq:traj_net} as the \textit{single-neural-net} parameterization. 

\begin{figure}[!htb]
	\centering
	\includegraphics[width=0.5\linewidth]{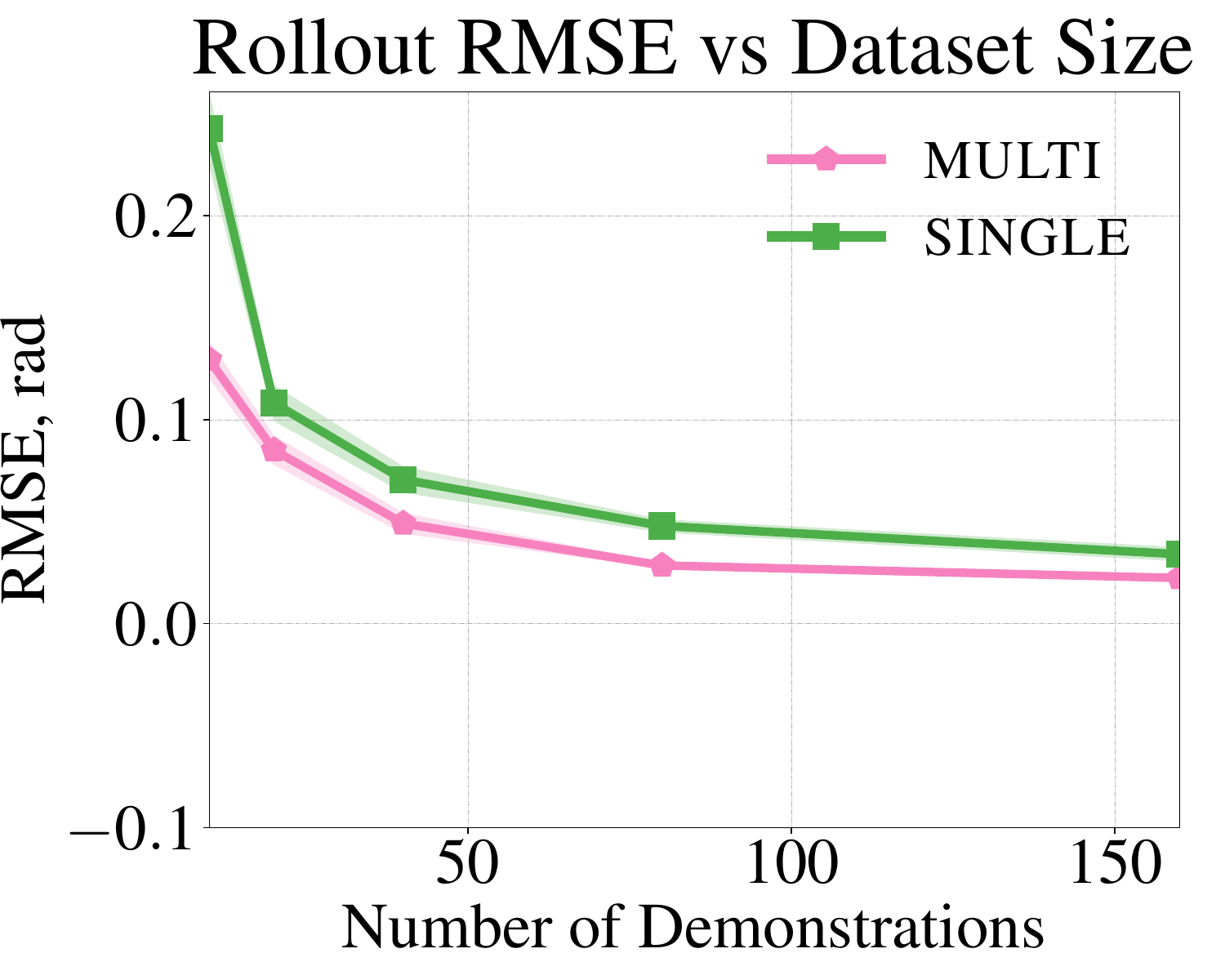}
	\caption{Rollout mean RMSE on the testing set with policies trained on 10 to 160 trajectories for Task $1$ over $6$ random seeds. Policies trained with both parameterizations achieve small error with datasets of size greater than 20. When trained with datasets of less than 20 trajectories, the learned policies under {single-neural-net} parameterization perform better.
	}
	\label{fig:single_vs_multi}
\end{figure}

\textbf{Results:} We now compare the results of both trajectory parameterizations to assess changes in performance. We focus on examining the generalization capability of CoDE to trajectories not seen during training under each auxiliary trajectory parameterization. 
We train the neural network policy with CoDE on 10 to 160 trajectories for Task $1$ over $6$ random seeds with the same training setup as listed in~\cref{sec:experiments}, and test it on the same set of 50 testing trajectories. As shown in Fig.~\ref{fig:single_vs_multi}, policies obtained with CoDE under both \textit{single-neural-net} and \textit{multi-neural-net} parameterizations achieve small error among trajectories with 20 or more demonstrations. When the number of demonstrations is less than 20, the \textit{multi-neural-net} parameterization performs better than the \textit{single-neural-net} parameterization. The better performance of the \textit{multi-neural-net} parameterization gives us an additional way to improve the performance of CoDE.


\section{Details of the Data Collection Setup}\label{sec:data_collection}

We give some details on the expert system from which we collected the demonstration data. As briefly described in Section~\ref{sec:experiments}, our expert is a state machine driven RMP system. We focus primarily on what it means to be driven by a state machine and what state machine we use. The RMP implementation we used matches that described by \cite{Cheng18rmpflow} in their appendix. At a high-level, that RMP setup implements a low-level motion generation system that drives the end-effector to targets while avoiding obstacles. We implement the expert by sequencing targets for the end-effector, either moving directly toward the target or approaching the target from a particular approach direction.

We implement approaching the target from a specified approach direction using a simple rule for defining a target offset as a function of orthogonal distance to the approach line. Let $\vv$ be the normalized approach direction with $\|\vv\| = 1$ defining the direction along which we should approach target $\x_g$. The approach behavior is defined by two parameters, a standoff length $l>0$ and a {\em standoff funnel} standard deviation parameter $\sigma>0$ defining the rule for how the standoff target $\x_o$ descends toward the true target $\x_g$. The rule is defined as follows
\begin{align}
    \x_o = \eta \x_g - (1-\eta) l\vv
    \ \ \ 
    \mbox{where}\ \ \
    \eta = \exp\left(-\frac{1}{2\sigma^2}\delta\x^\t\Big(\I-\vv\vv^\t\Big)\delta\x\right)
\end{align}
with $\delta\x = \x_g - \x$. In words, $\eta$ is a radial basis function of the distance from end-effector position $\x$ to the line defined by $\x_g$ and $\vv$ using standard deviation $\sigma$. When that orthogonal distance is large $\eta\approx 0$, and when that orthogonal distance is small $\eta\approx 1$. So far away, $\x_o$ is offset by length $l$ backward along $\vv$, but $\x_o$ starts descending toward $\x_g$ as the system approaches the target.

We implement the behavior using these two primitives (moving toward a target and moving toward the target along a particular approach direction). Each expert motion uses just two states:
\begin{enumerate}
    \item Lifting: Starting from the tabletop, the system sends the end-effector toward a point directly above it offset approximately $3h$ where $h>0$ is the desired height. Rather than waiting for the system to reach that point, it transitions from this state when simply it has reached a height of $h$. This prevents the system from converging to a stop at the target.
    \item Approaching: Now that it is at least $h$ meters above the table, we can send it to the target $\x_g$ approaching from above $\vv = (0,0,-1)$ using a standoff distance of $l = h$ and approach funnel standard deviation $\sigma$.
\end{enumerate}

For these experiments we could ignore end-effector orientation constraints since the default configuration automatically postured the arm so that the end-effector remained level to the surface on these problems.

\end{document}